\ifdefined\pdfminorversion\pdfminorversion=7\fi
\documentclass{article}
\usepackage{arxiv}
\usepackage[numbers,sort&compress]{natbib}

\date{}
\renewcommand{\shorttitle}{}

\usepackage[utf8]{inputenc} 
\usepackage[T1]{fontenc}    
\usepackage[hypertexnames=false]{hyperref}       
\hypersetup{hidelinks,pdfauthor={Kaixiang Shu}}
\usepackage{url}            
\usepackage{booktabs}       
\usepackage{amsfonts}       
\usepackage{nicefrac}       
\usepackage{microtype}      
\usepackage{xcolor}         
\usepackage{amsmath}
\usepackage{amssymb}
\usepackage{graphicx}
\usepackage{amsthm} 
\usepackage{amsmath} 
\usepackage{algorithm,algpseudocode}
\usepackage{placeins}
\usepackage{needspace}
\newtheorem{proposition}{Proposition}
\newtheorem{remark}{Remark}
\newtheorem{theorem}{Theorem}
\newtheorem{corollary}{Corollary}
\newtheorem{lemma}{Lemma}

\title{Adjoint Inversion Reveals Holographic Superposition and Destructive Interference in CNN Classifiers}

\author{%
  Kaixiang Shu \\
  Independent Researcher \\
  \texttt{614729197@qq.com} \\
}

\begin{document}

\maketitle

\begin{abstract}
How do CNN channels combine scene information into a class-specific response? We develop an input-conditioned adjoint inversion framework that renders individual channels and their combinations in pixel space. Lightweight Local Adjoint Correctors (LACs) calibrate backward signals through a frozen encoder, separating source-channel magnitude from spatial structure. Across four CNN backbones, individual deep-channel inversions retain global scene content, exhibiting a holographic organization that we describe as shared-scene superposition. Positive- and negative-weight class reconstructions each retain foreground and background content, yet their signed sum concentrates on the foreground. A decomposition into shared content and channel-specific residuals explains this recombination through destructive interference. Per-image spectra and reconstruction controls characterize this organization, while region-ranked channel ablations connect its spatial structure to class predictions. This account motivates covariance-volume channel selection (CovVol), which selects channels by the volume spanned by their centered global-average-pooled responses. We establish an exact squared-volume interpretation of the covariance determinant and derive its greedy selection rule. Compression, corruption, and transfer experiments demonstrate the practical value of this geometry. Applying LAC to a pure ViT-B/16 also improves input alignment over untrained LAC and raw vector--Jacobian products.
\end{abstract}

\section{Introduction}
\label{sec:intro}

Saliency and attribution methods, including Grad-CAM and its successors~\cite{selvaraju2017grad,chattopadhay2018grad,wang2020score} and perturbation-based approaches~\cite{petsiuk2018rise,fong2017interpretable}, make CNN decisions visually accessible. Their foreground-focused maps identify where a class response is concentrated, but leave open how scene content is represented across channels and combined by the classifier. We investigate two questions: what spatial information remains in individual deep-channel inversions, and how positive- and negative-weight channel contributions combine into a class-specific reconstruction.

Answering these questions requires instance-specific access to individual channels and their combinations in pixel space. Deconvolutional visualization projects feature activations back to pixel space~\cite{zeiler2014visualizing}. Optimization-based inversion~\cite{mahendran2015understanding} and learned decoders~\cite{dosovitskiy2016inverting,nguyen2017plug} reconstruct images using image priors or learned spatial synthesis. CAM-style methods summarize class relevance as heatmaps, while feature visualization~\cite{olah2017feature} exposes channel preferences through optimized inputs. We develop a complementary probe that follows a frozen encoder's input-conditioned adjoint, with trainable corrections restricted to per-channel scalar calibration.

Our framework separates channel magnitude from spatial structure. Forward activations define normalized weights $E_{l,c}$, and channel-selective vector--Jacobian products (VJPs) initialize the backward reconstructions. Local Adjoint Correctors (LACs), implemented with per-channel GroupNorm~\cite{wu2018group}, normalize and calibrate backward signals at CNN stage interfaces and the stem. The frozen encoder Jacobians supply spatially varying pullbacks, while the learned parameters apply spatially uniform scales and offsets. The resulting channel inversions form the input reconstruction $\widehat X_l=\sum_cE_{l,c}\widetilde V_{l,c}$. One deepest-stage objective trains the LAC parameters for each dataset--backbone classifier, and the same modules are reused across source layers. After training, these channel inversions also support class-directional combinations using the classifier's signed coefficients.

The probe reveals shared scene content across deep CNN channel inversions. Individual inversions retain global scene structure while differing in texture, color, and regional energy. We call this organization \emph{shared-scene superposition}: common input-aligned scene content coexists with channel-specific residuals. The term \emph{holographic} describes the global scene content retained across individual channel inversions. Per-image spectral measurements and matched controls quantify the common component, while random and sparsity controls distinguish spatial diffuseness from reconstruction fidelity. When channels are combined according to classifier coefficients, both the positive- and negative-weight components retain scene content, yet their signed sum concentrates on the foreground.

We explain this recombination through destructive interference. Writing $\widetilde V_i(X)=C_i(X)B(X)+\delta_i(X)$ separates the shared scene direction from channel-specific residuals. For a fixed input and inversion basis, the class reconstruction is linear in the classifier weights. Its shared component is multiplied by $\sum_iw_i^{(c)}C_i$, so signed combinations can suppress common content while combining discriminative residual structure. Region-ranked channel ablations connect these spatial preferences to class predictions, complementing the visual and spectral evidence with channel interventions.

This account also motivates selecting channels with substantial response magnitude and complementary directions. To select a common subset across images, we consider each channel's global average pooling (GAP) responses over a dataset. The covariance determinant measures the squared volume spanned by these centered responses, capturing both their scale and directional diversity. We derive a greedy selection rule using pivoted Cholesky and evaluate the resulting covariance-volume criterion, CovVol, through compression and corruption experiments.

We also apply LAC to a pure ViT-B/16, placing correctors at Transformer block boundaries and including a stem corrector. This adaptation improves input alignment over untrained LAC and raw VJP (Section~\ref{sec:pure_vit}). A hybrid attention-head experiment further demonstrates reuse of a trained CNN cascade for downstream readouts (Appendix~\ref{app:hybrid_attention}).

Our contributions are: (i) an input-conditioned adjoint inversion framework with lightweight scalar calibration, explicit spatial and algebraic properties, and an application to a pure ViT; (ii) direct evidence of shared-scene superposition in CNN channel inversions and a destructive-interference account of their class-weighted combinations; and (iii) region-guided channel ablation and covariance-volume selection that connect the observed structure to classifier behavior and representation compression.

\section{Method}
\label{sec:method}

\subsection{Notation and Overview}
\label{sec:notation}
Let $X\in\mathcal X\subseteq\mathbb R^{3\times H_{\rm in}\times W_{\rm in}}$ be an input image. A frozen encoder contains a stem $h_{\rm stem}(X)$ and $L$ feature stages $h_l(X)\in\mathbb R^{C_l\times H_l\times W_l}$, indexed by $l=0,\ldots,L-1$. The spatial plane $h_{l,c}(X)$ is channel $c$ of stage $l$. We write
\begin{equation}
J_{\rm stem}=\frac{\partial h_{\rm stem}}{\partial X},\qquad
J_0=\frac{\partial h_0}{\partial h_{\rm stem}},\qquad
J_b=\frac{\partial h_b}{\partial h_{b-1}}\quad(b\geq1),
\label{eq:jacobian-definitions}
\end{equation}
with all Jacobians evaluated at the forward activations of $X$. Their transposes act on backward signals through reverse-mode automatic differentiation, without explicitly forming the Jacobian matrices~\cite{JMLR:v18:17-468}. The operator $\Psi_{\mathrm{LAC}_b}$ applies a learned per-plane GroupNorm to the output tensor of $J_b^T$; $\Psi_{\mathrm{LAC}_{\rm stem}}$ acts on the final RGB tensor.

The procedure has five steps: compute forward channel weights, seed one source channel with its activation, propagate that seed through the frozen local adjoints and LACs, obtain its calibrated inversion, and combine channel inversions with the forward weights. Algorithm~\ref{alg:lac_inversion} gives the full synthesis. Only the GroupNorm scale and offset parameters are trained; the encoder and its local derivatives remain frozen.

\subsection{Why Calibrate the Raw VJP?}
\label{sec:crises}
\textbf{Zero-insertion artifacts.} CNNs reduce spatial resolution by stride $s\geq2$. The adjoint of this sampling operation inserts zeros before applying the transposed convolution, which can produce periodic artifacts when the interpolation induced by the learned kernel is uneven~\cite{odena2016deconvolution}. This known effect motivates calibration at stage interfaces and the stem; Supplementary~A.1 describes the sampling operation. The placement check in Table~\ref{tab:lac_placement} compares internal-boundary correction with correction that also includes the stem.

\textbf{Magnitude--direction coupling.} Let $v\in\mathbb R^d$ be a nonzero flattened reference adjoint signal at a boundary, and let $v+\epsilon$ be a perturbed signal on the same grid. Write $\epsilon=\epsilon_\parallel+\epsilon_\perp$, with components parallel and orthogonal to $v$. A small reference signal can have a substantially altered normalized direction. Specifically, for $K\geq0$,
\begin{equation}
\left\|\frac{v+\epsilon}{\|v+\epsilon\|}-\frac{v}{\|v\|}\right\|
\geq\frac{1}{\sqrt{(1+K)^2+1}}
\quad\text{if}\quad
\|\epsilon_\perp\|\geq\|v\|,\quad
\|\epsilon_\parallel\|\leq K\|v\|.
\label{eq:angle-collapse}
\end{equation}
This conditional vector bound motivates treating response magnitude separately from spatial direction; its proof is in Supplementary~A.1. We use \emph{spatial direction} for the normalized pattern of a signal.

These considerations lead to complementary roles for the forward and backward paths. Forward weights preserve relative source-channel magnitudes, while the backward cascade calibrates spatial planes at the stage interfaces and stem. The frozen Jacobians provide the spatially varying pullbacks along this cascade.

\subsection{Forward Path and Backward Path}
\label{sec:forward}
The forward path assigns each source channel a nonnegative weight using global average pooling (GAP) of its absolute activation:
\begin{equation}
Z_{l,c}=\mathrm{GAP}(|h_{l,c}|)=\frac{1}{H_lW_l}\sum_{u,v}|h_{l,c}(u,v)|,
\qquad E_{l,c}=\frac{Z_{l,c}}{\sum_{c'}Z_{l,c'}}.
\label{eq:Z-E-def}
\end{equation}
For $\sum_{c'}Z_{l,c'}>0$, the weights sum to one and lie on the simplex $\Delta^{C_l-1}$. They express relative activation magnitude across channels. Signed responses are retained in the backward seed, while absolute values make the forward allocation nonnegative (Supplementary~A.2).

For a source channel $c$, let $e_c$ select that channel and broadcast over its spatial plane. The channel-selective seed and raw VJP are
\begin{equation}
\mathrm{seed}_{l,c}=e_c\odot h_l,\qquad
\mathrm{VJP}_{l,c}=\left(\frac{\partial h_l}{\partial X}\right)^T\mathrm{seed}_{l,c}.
\label{eq:vjp-seed}
\end{equation}
The chain rule gives the exact activation-energy identity
\begin{equation}
\mathrm{VJP}_{l,c}=\nabla_X\Phi_{l,c}(X),\qquad
\Phi_{l,c}(X)=\tfrac12\|h_{l,c}(X)\|_2^2.
\label{eq:vjp-energy}
\end{equation}
Each response contributes its derivative weighted by its activation. This identity applies to signed as well as nonnegative activations. The global adjoint factorizes as
\begin{equation}
\left(\frac{\partial h_l}{\partial X}\right)^T
=J_{\rm stem}^T\circ J_0^T\circ\cdots\circ J_l^T.
\label{eq:markov-decomp}
\end{equation}
The factors include the encoder's frozen intra-stage operations; calibration acts at the interfaces of this cascade.

\subsection{LAC at Stage Interfaces and the Stem}
\label{sec:lac}
Each LAC uses GroupNorm with $\texttt{num\_groups}=\texttt{num\_channels}$. Thus each spatial plane of the backward tensor is normalized separately. For a plane $v\in\mathbb R^n$, let $P=I-\mathbf1\mathbf1^T/n$. Its output is
\begin{equation}
\mathrm{GN}_c(v)=\gamma_c\frac{Pv}{\sqrt{\|Pv\|_2^2/n+\varepsilon}}+\beta_c\mathbf1.
\label{eq:gn-strip}
\end{equation}
Here $c$ indexes the channels of the tensor at the current interface; at the final stem correction these are the RGB planes. The parameters are shared across source channels and source layers that traverse that interface. Centering removes the spatial mean, variance normalization controls scale, and the affine parameters apply a spatially uniform scale and offset. With $\varepsilon=0$ and $Pv\ne0$, the standardized vector has norm $\sqrt n$; positive $\varepsilon$ stabilizes constant or low-variance planes. No learned spatial kernel is introduced.

The corrected inversion is
\begin{equation}
\widetilde V_{l,c}=
\Psi_{\mathrm{LAC}_{\rm stem}}\circ J_{\rm stem}^T\circ
\Psi_{\mathrm{LAC}_0}\circ J_0^T\circ\cdots\circ
\Psi_{\mathrm{LAC}_l}\circ J_l^T(\mathrm{seed}_{l,c}).
\label{eq:lac-cascade}
\end{equation}
This calibrated spatial basis supplies the channel's rendering, while $E_{l,c}$ retains its forward magnitude allocation. Supplementary~A.3 describes the cascade topology, and Appendix~B derives its affine and spatial properties.

\subsection{Synthesis, Training, and Inference}
\label{sec:synthesis}
The layer reconstruction combines weights and bases from the same source stage:
\begin{equation}
\widehat X_l=\sum_{c=1}^{C_l}E_{l,c}\widetilde V_{l,c}.
\label{eq:synthesis}
\end{equation}
The contribution of channel $c$ is $\widehat X_{l,c}=E_{l,c}\widetilde V_{l,c}$. The same trained cascade produces individual channel visualizations and the full synthesis.

\begin{algorithm}[t]
\caption{Channel Inversion and Layer Reconstruction}\label{alg:lac_inversion}
\begin{algorithmic}[1]
\State \textbf{Input:} image $X$, source stage $l$, frozen encoder, and LAC parameters.
\State Compute the forward activations and $E_{l,c}$ by Eq.~\eqref{eq:Z-E-def}.
\For{$c=1,\ldots,C_l$}
\State $z\gets e_c\odot h_l$.
\For{$b=l,l-1,\ldots,0$}
\State $z\gets\Psi_{\mathrm{LAC}_b}(J_b^Tz)$.
\EndFor
\State $\widetilde V_{l,c}\gets\Psi_{\mathrm{LAC}_{\rm stem}}(J_{\rm stem}^Tz)$.
\EndFor
\State \textbf{Output:} channel bases $\{\widetilde V_{l,c}\}_c$ and $\widehat X_l=\sum_cE_{l,c}\widetilde V_{l,c}$.
\end{algorithmic}
\end{algorithm}

We train one LAC checkpoint for each dataset--backbone classifier using the deepest-stage reconstruction objective
\begin{equation}
\mathcal L=\|X-\widehat X_{L-1}\|_1.
\label{eq:loss}
\end{equation}
Only the LAC affine parameters are updated. The deepest path traverses every LAC module; at inference, a shallower source uses a subset of those same modules. The checkpoint is therefore reused across source layers and inversion analyses of the same model without retraining.

Training uses Adam~\cite{kingma2014adam} with learning rate $10^{-3}$ for $10{,}000$ steps, with trainable-parameter gradient norms clipped at $0.01$. Each step samples eight source channels with replacement according to the batch-mean activation-energy distribution. Batch size is 20 for ImageNet and 24 for the other datasets. Weight decay is $10^{-5}$ for ImageNet and CIFAR-100 and $10^{-4}$ for CUB, Pets, and Dogs. We use the official training splits, ImageNet normalization for ImageNet, and per-channel mean and standard deviation $0.5$ for the other datasets. The reconstruction objective contains no total-variation term, natural-image prior, generative prior, or post-hoc smoothing. Supplementary~A.4 summarizes training and cross-layer reuse.

\subsection{Adjoint Structure and Spatial Fidelity}
\label{sec:guarantees}
The construction separates the encoder's spatial propagation from the learned plane-wise calibration. The following results characterize these two operations through raw-adjoint support and the local spatial action of LAC.
\begin{theorem}[Raw-Adjoint Spatial Fidelity]
\label{thm:spatial-fidelity}
At an input where the encoder Jacobian is defined, let
\begin{equation*}
\mathrm{EF}_{l,c}(X)=\bigcup_{q:h_{l,c,q}(X)\ne0}\mathrm{supp}(\nabla_Xh_{l,c,q}(X)).
\end{equation*}
Then the raw channel VJP satisfies
\begin{equation}
\mathrm{supp}(\mathrm{VJP}_{l,c})\subseteq\mathrm{EF}_{l,c}(X).
\label{eq:spatial-fidelity}
\end{equation}
\end{theorem}
\noindent\emph{Proof sketch.} Expanding Eq.~\eqref{eq:vjp-energy} gives $\mathrm{VJP}_{l,c}=\sum_qh_{l,c,q}\nabla_Xh_{l,c,q}$. A coordinate outside the stated union receives no nonzero term. This argument applies to GELU and Swish as well as ReLU wherever the relevant derivatives exist; it does not require binary activation gates. Supplementary~B.4 gives the full proof.

\begin{corollary}[Local Spatial Action of LAC]
\label{cor:local-spatial-action}
For a spatial difference operator $D$ satisfying $D\mathbf1=0$, each LAC plane obeys
\begin{equation}
D\,\mathrm{GN}_c(v)=\frac{\gamma_c}{\sqrt{\|Pv\|_2^2/n+\varepsilon}}\,Dv.
\label{eq:local-spatial-action}
\end{equation}
\end{corollary}
Each correction uniformly rescales spatial differences within its current plane. Spatially varying pullbacks follow the frozen encoder Jacobians, while LAC performs plane-wise normalization and a uniform affine transformation. Appendix~B establishes the full affine decomposition for signed scales and nonzero biases at any training iterate, a stem-support result for the complete cascade, and the shared-module properties used for cross-layer inference.

\section{Experiments}
\label{sec:experiments}

\subsection{Reconstruction Fidelity Against Inversion Baselines}
\label{sec:exp_comparison}

We first evaluate input reconstruction by combining the calibrated channel inversions with their forward activation weights. We compare four baselines inverting the frozen encoder's deepest features: \textbf{Vanilla Gradient} (raw VJP), \textbf{Mahendran \& Vedaldi}~\cite{mahendran2015understanding} (TV/$\ell_2$ opt.), \textbf{Robust Inversion}~\cite{engstrom2019adversarial} (robust), and \textbf{UpConvNet}~\cite{dosovitskiy2016inverting} (feedforward decoder). Table~\ref{tab:inversion_comparison} reports SSIM~\cite{Zhou_Wang_2004} ($\uparrow$) and LPIPS~\cite{Zhang_2018_CVPR} ($\downarrow$) across five datasets (ImageNet\cite{russakovsky2015imagenet}, CIFAR-100\cite{krizhevsky2009learning}, CUB-200\cite{wah2011caltech}, Pets\cite{parkhi2012cats}, Dogs\cite{khosla2011novel}) and four backbones: \textit{R-50/R-18/D-121/C-Base}.

\begin{table}[t]
\centering
\caption{Reconstruction fidelity: SSIM and LPIPS multiplied by 100. Format \textit{R-50/R-18/D-121/C-Base}; bold marks the best value for each backbone and metric.}
\label{tab:inversion_comparison}
\resizebox{\textwidth}{!}{
\begin{tabular}{lcccccc}
\toprule
Method & Metric & ImageNet & CIFAR-100 & CUB-200 & Pets & Dogs \\
\midrule
\textbf{Ours} & SSIM $\uparrow$ & \textbf{41.7/34.6/38.5/43.2} & 67.6/36.8/40.4/69.6 & \textbf{54.8/44.2/40.7/47.7} & \textbf{52.1}/35.4/34.9/37.6 & \textbf{51.7/36.1/34.8/38.1} \\
& LPIPS $\downarrow$ & \textbf{43.0/57.7/45.8/53.5} & 85.4/93.8/94.6/64.9 & \textbf{52.9/64.4/61.5/64.6} & \textbf{54.3/72.7/63.8/71.5} & \textbf{51.5/72.1/63.5/70.0} \\
Vanilla & SSIM $\uparrow$ & 4.8/2.5/4.3/3.7 & 4.5/2.0/4.1/7.6 & 5.0/3.1/4.9/6.1 & 3.5/1.9/3.0/4.1 & 5.1/3.2/4.7/5.2 \\
& LPIPS $\downarrow$ & 109.5/119.0/107.9/108.8 & 140.8/173.0/144.6/113.8 & 129.1/140.4/133.9/134.1 & 124.9/134.3/119.6/129.9 & 128.0/139.6/123.6/125.7 \\
Mahendran & SSIM $\uparrow$ & 7.5/6.2/6.9/8.7 & 13.8/7.3/9.8/13.2 & 10.5/8.4/6.2/12.5 & 13.2/7.5/9.6/13.9 & 10.5/6.5/8.2/14.1 \\
& LPIPS $\downarrow$ & 99.0/107.0/108.8/104.5 & 136.4/164.0/142.7/72.3 & 118.6/133.2/127.3/95.1 & 113.3/125.0/112.2/87.9 & 115.6/131.9/118.3/81.3 \\
Robust & SSIM $\uparrow$ & 7.6/6.5/8.7/9.6 & 8.7/7.4/10.1/12.5 & 7.1/4.7/4.9/8.9 & 7.4/4.7/6.3/9.7 & 6.2/4.4/5.2/10.1 \\
& LPIPS $\downarrow$ & 105.2/113.6/116.8/101.6 & 142.7/165.4/143.0/88.0 & 121.0/132.1/131.5/105.8 & 119.2/124.1/121.7/106.0 & 120.8/131.4/123.9/99.6 \\
UpConvNet & SSIM $\uparrow$ & 15.9/18.5/14.5/12.0 & \textbf{69.8/63.4/67.4/71.2} & 43.5/36.4/38.5/36.2 & 43.5/\textbf{39.6}/\textbf{41.9}/\textbf{43.1} & 36.4/32.3/33.6/37.4 \\
& LPIPS $\downarrow$ & 89.3/92.1/92.2/89.0 & \textbf{62.8/73.8/75.4/63.5} & 64.1/76.8/71.2/87.0 & 76.5/82.6/77.9/76.0 & 84.4/91.4/87.0/85.2 \\
\bottomrule
\end{tabular}}
\end{table}

\begin{figure}[t]
\centering
\begin{minipage}{0.60\textwidth}
\centering\includegraphics[width=\linewidth]{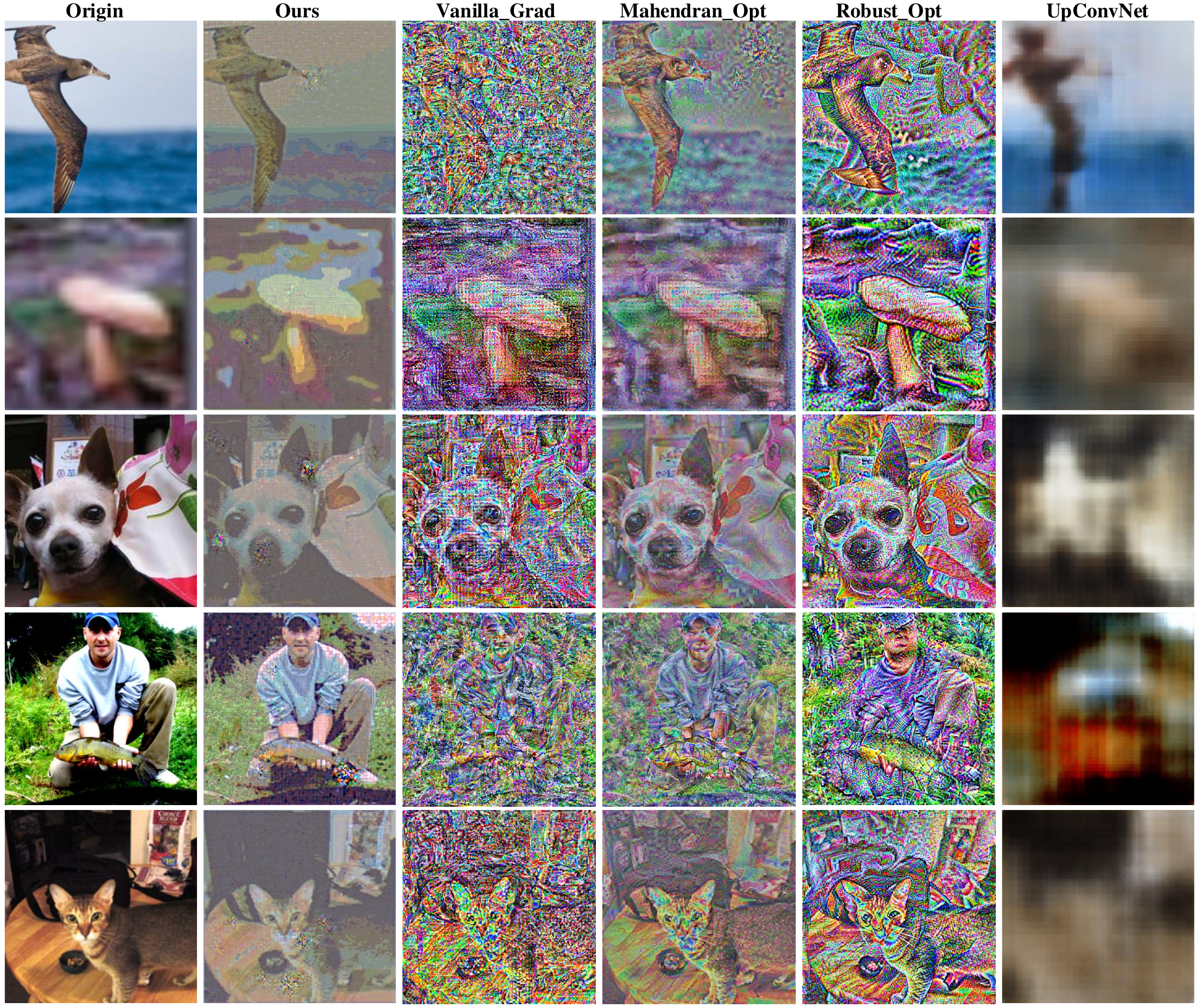}
\caption{Qualitative inversion comparison; Our framework preserves structural details and sharp contours.}
\label{fig:qualitative_comparison}
\end{minipage}  \quad
\begin{minipage}{0.32\textwidth}
\centering\includegraphics[width=\linewidth]{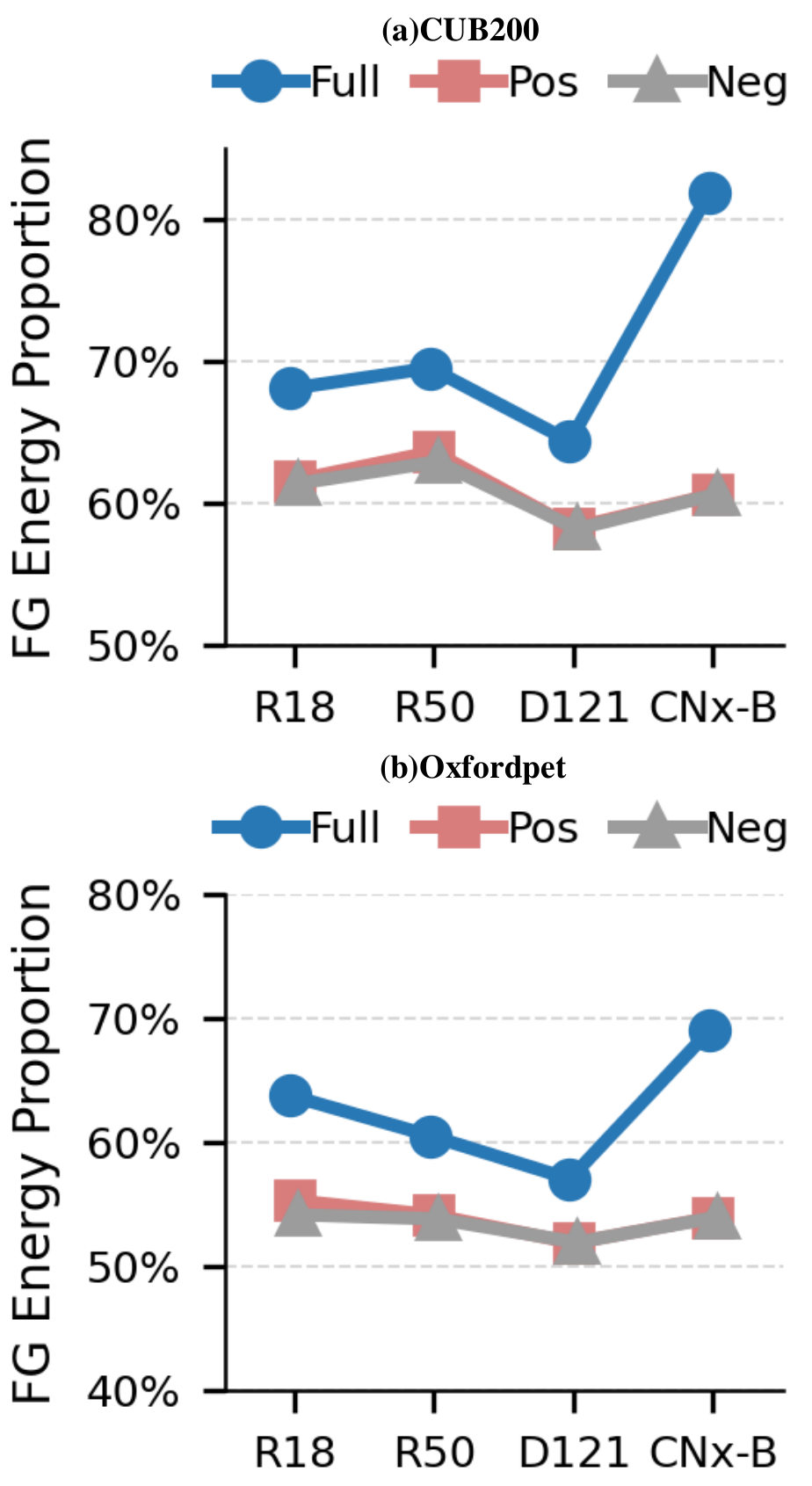}
\caption{Foreground energy proportion $\mathrm{FG}$.} 
\label{fig:fg_energy}
\end{minipage}
\end{figure}

Our method achieves the best LPIPS on ImageNet, CUB-200, Pets, and Dogs for every evaluated backbone, and the best SSIM on ImageNet, CUB-200, and Dogs. On Pets, it leads SSIM for ResNet-50, while UpConvNet leads for the other three backbones. On ImageNet, our SSIM is $1.9$--$3.6\times$ UpConvNet's, with LPIPS reduced by approximately $37$--$52\%$. Vanilla Gradient has low reconstruction SSIM ($\leq7.6$), while the calibrated cascade recovers substantially clearer structure. UpConvNet performs best on CIFAR-100 under the enlarged-image protocol of Table~\ref{tab:inversion_comparison}; the native-resolution and evaluation-scale comparisons below examine this dependence. Figure~\ref{fig:qualitative_comparison} shows the corresponding visual differences (Supplementary~C.1).

\paragraph{CIFAR-100 at native and enlarged resolutions.}
\label{sec:cifar_protocol}
At native $32\times32$ resolution, LAC outperforms UpConvNet on CIFAR-100 with ResNet-18: SSIM increases from $0.197$ to $0.703$ and LPIPS decreases from $0.616$ to $0.423$ (Table~\ref{tab:cifar_protocol}). The enlarged bilinear protocol gives a different ranking. When $32\times32$ inputs are enlarged to $224\times224$ and evaluated at that size, UpConvNet has higher SSIM and lower LPIPS. Evaluating the same models and outputs at $32\times32$ reverses both rankings. The comparison therefore depends on the evaluation scale, with LAC favored at the native image scale.

\begin{table}[t]
\centering\small
\caption{CIFAR-100 reconstruction with ResNet-18 under three protocols. Values are mean$\pm$SD across three seeds in raw metric units; Table~\ref{tab:inversion_comparison} scales both metrics by 100. The first two settings use the same models and outputs, changing only the evaluation size. Bold marks the better method within each protocol.}
\label{tab:cifar_protocol}
\setlength{\tabcolsep}{5pt}
\begin{tabular}{@{}p{0.37\linewidth}lcc@{}}
\toprule
Protocol & Method & SSIM $\uparrow$ & LPIPS $\downarrow$ \\
\midrule
Bilinear $32\!\to\!224$, evaluate at $224$ & LAC & $0.329\pm0.003$ & $0.952\pm0.004$ \\
& UpConvNet & $\mathbf{0.603\pm0.003}$ & $\mathbf{0.677\pm0.008}$ \\
\midrule
Same models and outputs, evaluate at $32$ & LAC & $\mathbf{0.444\pm0.001}$ & $\mathbf{0.476\pm0.000}$ \\
& UpConvNet & $0.275\pm0.004$ & $0.560\pm0.008$ \\
\midrule
Native $32\times32$ & LAC & $\mathbf{0.703\pm0.010}$ & $\mathbf{0.423\pm0.010}$ \\
& UpConvNet & $0.197\pm0.006$ & $0.616\pm0.005$ \\
\bottomrule
\end{tabular}
\end{table}

The CIFAR-100 entries in Table~\ref{tab:inversion_comparison} report single-run results; Table~\ref{tab:cifar_protocol} reports separate three-seed evaluations. For the enlarged protocol, the trained encoder is fixed and each seed is evaluated on 2,000 test images. The native protocol uses three classifier/LAC/decoder seeds, each evaluated on all 10,000 test images. These are separate training protocols. Supplementary~\ref{app:exp_C1} gives the fixed-checkpoint interpolation controls and their paired image-level confidence intervals.

\subsection{Shared-Scene Superposition and Destructive Interference}
\label{sec:exp_paradigm}

Having evaluated input reconstruction, we examine what individual channel inversions retain and how their class-specific combinations change the spatial distribution of energy.

\subsubsection{The Foreground Energy Paradox}
\label{sec:fg_paradox}

\paragraph{Class-directional inversion.}
The input reconstruction uses activation-based weights $E_{L-1,i}$ to assemble the calibrated channel bases. To examine a class direction, we reuse these bases from the same trained cascade and combine them with the classifier coefficients $w_i^{(c)}$.
Let $h_i(X)=\mathrm{GAP}(\mathbf h_{L-1,i}(X))$ and $w_i^{(c)}=\partial\mathrm{logit}_c/\partial h_i$; for a linear classifier these are its learned weights. Holding the input and channel inversions fixed, we define the class-directional reconstruction by
\begin{equation}
\hat{X}^{(c)}(X) = \sum_{i=1}^{C_{L-1}} w_i^{(c)}\,\tilde{V}_{L-1,i}(X), \qquad \hat{X}^{(c)} = {\sum_{w_i^{(c)}>0}\!\! w_i^{(c)}\,\tilde{V}_{L-1,i}} + {\sum_{w_i^{(c)}<0}\!\! w_i^{(c)}\,\tilde{V}_{L-1,i}},
\end{equation}
For fixed $X$, synthesis is linear in the coefficients $w^{(c)}$. The split into positive- and negative-weight components is therefore exact; we denote these components by $\hat X^{(c)}_+$ and $\hat X^{(c)}_-$.

\paragraph{Foreground energy proportion.}

Given reconstruction $\hat{X}$ and ground-truth mask $M_{\mathrm{fg}}$, define
$\mathrm{FG}(\hat{X}) := \sum_{(i,j)\in M_{\mathrm{fg}}}\!\!\hat{X}(i,j)^2 / \sum_{(i,j)} \hat{X}(i,j)^2 \in [0,1].$
We evaluate on \textbf{CUB-200-2011}~\cite{wah2011caltech} and \textbf{Oxford-IIIT Pet}~\cite{parkhi2012cats} across four ImageNet-pretrained backbones (ResNet-18/50, DenseNet-121, ConvNeXt-Base), using the top-1 class $c$. Figure~\ref{fig:fg_energy} shows: \textbf{(O1)} similar foreground-energy proportions for the positive and negative components; \textbf{(O2)} substantially higher foreground concentration in their full sum; and \textbf{(O3)} the largest gain for ConvNeXt-Base on both datasets.

\textbf{The paradox.} Foreground concentration emerges from summation even though both sign components retain the scene:
$\mathrm{FG}(\hat X^{(c)}_+)\approx\mathrm{FG}(\hat X^{(c)}_-)\ll\mathrm{FG}(\hat X^{(c)})$.
This pattern motivates examining cancellation between the two components.

\begin{figure}[t]
\centering\includegraphics[width=0.94\linewidth]{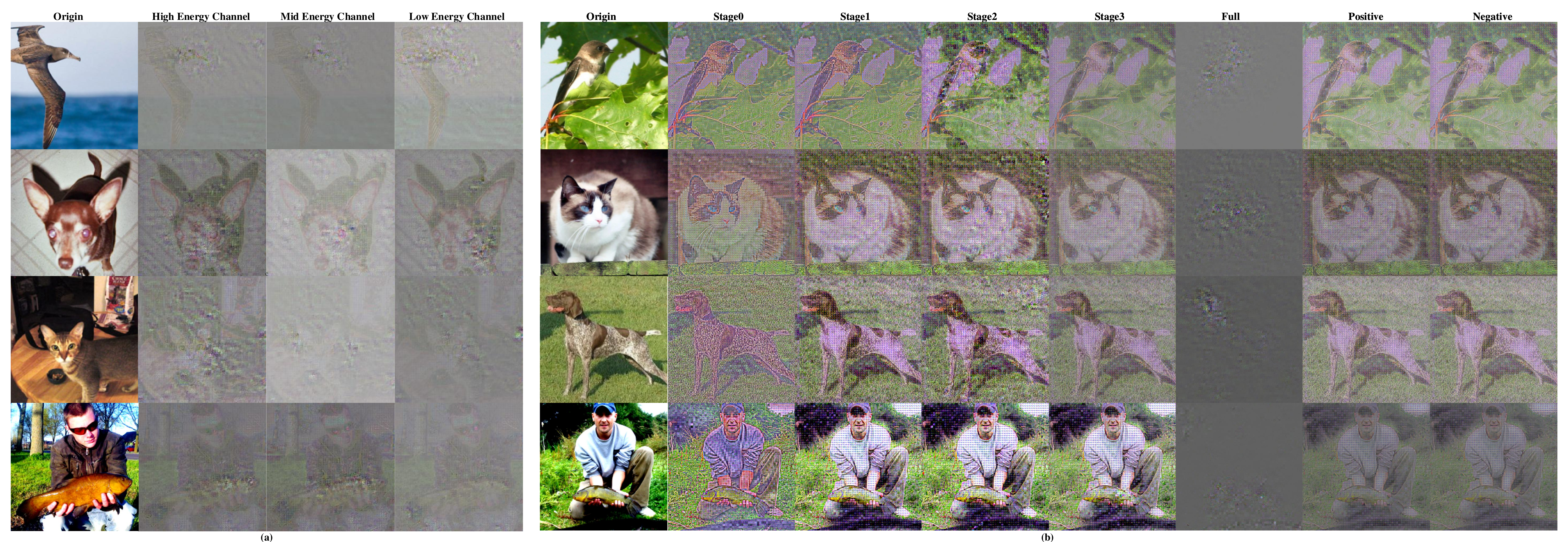}
\caption{(a) Per-channel inversions $\tilde V_{L-1,i}(X)$ across energy ranks retain global scene structure. (b) Stage-wise inversions $\hat X_{0\text{--}3}$ and class-weighted reconstructions: both positive- and negative-weight components retain foreground and background content, while their signed sum concentrates on the foreground.}
\label{fig:channel_holography}
\end{figure}

\paragraph{Scene content and signed recombination.}\label{sec:scene_composition}
Figure~\ref{fig:channel_holography} addresses the two questions posed in the Introduction. Across the displayed energy ranks, individual channel inversions retain global scene support while differing in texture, color, and regional energy. These global supports coexist with the local preferences tested in Section~\ref{sec:causal_ablation}. Across stages, the scene remains visible as its style and frequency content change. Both sign-restricted components contain foreground and background, yet their signed sum concentrates on the foreground. This exposes cancellation directly in the class-directional reconstruction.

\subsubsection{The Destructive-Interference Account}
\label{sec:new_paradigm}

The observations suggest two linked ingredients: common scene content across channels and cancellation when those channels are combined with signed coefficients. H1 describes the channel structure, and H2 describes its recombination along a class direction.

\paragraph{H1: Shared-Scene Superposition.}
Every deep channel inversion decomposes as
\begin{equation}
\label{eq:h1}
\tilde{V}_{L-1,i}(X) = C_i(X)\cdot B(X) + \delta_i(X), \qquad \langle \delta_i(X), B(X)\rangle \approx 0,
\end{equation}
where $B(X)$ is a shared scene direction, $C_i(X)$ its channel coefficient, and $\delta_i(X)$ a channel-specific residual. Shared-scene superposition denotes common input-aligned scene structure distributed across channel inversions together with their individual residuals. Per-image leading-mode energy fractions average $77.4\%$ on CUB/ConvNeXt-Base and $95.9\%$ on Pets/ResNet-50, substantially above matched null controls (Table~\ref{tab:scene_spectrum}). Channel-centering sharply reduces these fractions, identifying the common component as the main source of the uncentered spectral concentration. Supplementary~C.4 distinguishes these per-image measurements from the $20.9\%$ leading-mode fraction obtained by decomposing an averaged Gram matrix.

Random and sparsity controls separate spatial diffuseness from reconstruction fidelity (Tables~\ref{tab:lac_entropy} and~\ref{tab:lac_control_fidelity}). Untrained and random-affine LACs retain high spatial entropy but produce low-SSIM reconstructions. A weak sparsity objective concentrates spatial energy while retaining comparable reconstruction fidelity, whereas stronger penalties increase LPIPS. Standard trained LAC combines broad scene coverage with input-aligned reconstruction, while the sparsity controls show how the objective adjusts spatial concentration.

\paragraph{H2: Classification as Destructive Interference.}
Substituting H1 into the class-directional reconstruction yields the exact decomposition
\begin{equation}
\label{eq:h2}
\hat{X}^{(c)}(X) = B(X)\,\underbrace{\textstyle\sum_i w_i^{(c)} C_i(X)}_{\alpha^{(c)}(X)} + \underbrace{\textstyle\sum_i w_i^{(c)} \delta_i(X)}_{\text{foreground residual sum}}.
\end{equation}
In this decomposition, foreground-concentrated class reconstructions correspond to two conditions: \textbf{(B1)} suppression of the shared coefficient, $\alpha^{(c)}(X)\approx0$, and \textbf{(B2)} constructive combination of discriminative residuals. The positive and negative components can each retain common scene content, which is suppressed when the components are summed. This destructive-interference account explains O1--O2 and connects the visual signature to the classifier's signed coefficients (Supplementary~C.4).

\subsubsection{From Interference Geometry to Covariance-Volume Channel Selection}
\label{sec:covvol_algorithm}

In H1--H2, shared-content cancellation leaves a combination of channel-specific residuals. For a channel subset $S$, exact cancellation restricts the coefficients to $C_S(X)^\top w=0$; Supplementary~C.5 characterizes the available residual volume. This geometric view motivates retaining both response scale and directional diversity when selecting channels.

Channel selection requires a common subset across images, so we turn to the channels' GAP responses over a dataset. These forward responses retain their natural scales before LAC normalization and are directly combined by a linear classifier. Magnitude-based filter pruning ranks filters independently by their weight norms~\cite{li2016pruning}. Selecting channels by strength alone can retain highly redundant response directions. We therefore use covariance volume to favor subsets with both substantial response magnitude and complementary directions. This criterion connects to the use of Gram determinants to model quality and diversity in determinantal point processes~\cite{kulesza2012determinantal}. The following theorem gives this criterion an exact geometric interpretation.

\begin{theorem}[Covariance Volume of Linear Responses]
\label{thm:covvol_duality}
Let $H\in\mathbb R^{N\times C}$ contain GAP features, $\bar h=N^{-1}H^\top\mathbf1$, $A=N^{-1/2}(H-\mathbf1\bar h^\top)$, and $\Sigma=A^\top A$. For $|S|=k$, the selected channel-response parallelepiped satisfies
\begin{equation}\label{eq:covvol-main}
\operatorname{Vol}_k\bigl(A_S[0,1]^k\bigr)^2=\det(\Sigma_S).
\end{equation}
Equivalently, the image of the unit coefficient ball has volume $\kappa_k\sqrt{\det\Sigma_S}$, where $\kappa_k$ is the unit $k$-ball volume. The $k$-volume is zero when $A_S$ has rank below $k$.
\end{theorem}
\noindent\emph{Proof sketch.} If $s_1,\ldots,s_k$ are the singular values of $A_S$, its $k$-dimensional volume factor is $\prod_js_j$, while $\det(A_S^TA_S)=\prod_js_j^2$. The same factor scales the volume of the unit coefficient ball; rank deficiency makes both sides zero.
The vector $A_Sw$ is a linear head's centered score vector over the dataset, scaled by $N^{-1/2}$. Covariance volume therefore measures the range of centered score responses available from a subset under a fixed coefficient-norm budget. The determinant combines response scale with directional diversity, vanishing when the selected responses are linearly dependent. Supplementary~C.5 proves the theorem and the greedy update below.

\begin{algorithm}[t]
\caption{Covariance-Volume Channel Selection}\label{alg:covvol}
\begin{algorithmic}[1]
\State \textbf{Input:} GAP matrix $H\in\mathbb R^{N\times C}$; budget $k\leq\operatorname{rank}(\Sigma)$.
\State $A\gets N^{-1/2}(H-\mathbf1\bar h^\top)$; $\Sigma\gets A^\top A$; $R\gets\Sigma$; $S\gets\varnothing$.
\For{$t=1,\ldots,k$}
\State $i^\star\gets\arg\max_{i\notin S}R_{ii}$; $S\gets S\cup\{i^\star\}$.
\State $R\gets R-R_{:,i^\star}R_{i^\star,:}/R_{i^\star i^\star}$.
\EndFor
\State \textbf{Output:} $S$.
\end{algorithmic}
\end{algorithm}
\begin{figure}[t]
\centering\includegraphics[width=0.92\linewidth]{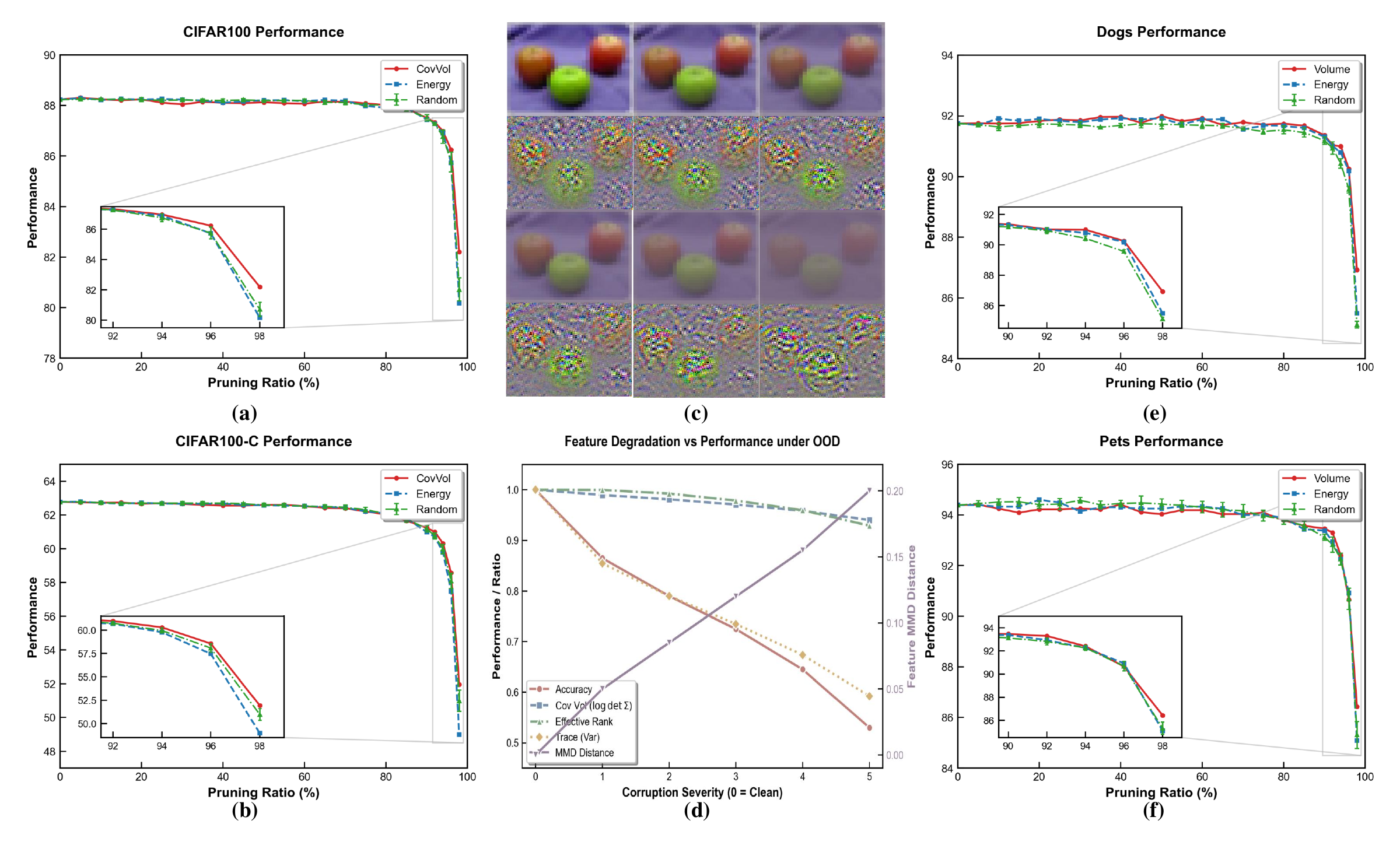}
\caption{(a,b,e,f) Channel pruning versus accuracy on CIFAR-100, CIFAR-100-C, Dogs, and Pets, with linear probes trained on retained features. (c) Class reconstructions under corruption. (d) Normalized accuracy, covariance statistics, effective rank, and MMD versus corruption severity.}
\label{fig:covvol_main}
\end{figure}

Maximizing $\det\Sigma_S$ for a fixed budget is NP-hard in general~\cite{civril2009selecting}. Algorithm~\ref{alg:covvol} uses diagonal-pivoted partial Cholesky, as in fast greedy determinant maximization~\cite{NEURIPS2018_dbbf603f}: each step selects the largest conditional residual variance, which is exactly the determinant ratio $\det\Sigma_{S\cup\{i\}}/\det\Sigma_S$. It therefore greedily selects channels by covariance volume, balancing response magnitude against redundancy.

\subsubsection{Empirical Validation: Compression and the Geometry of OOD Failure}
\label{sec:covvol_validation}

We test whether selecting complementary response directions preserves classification performance at small channel budgets, and how response geometry changes under corruption. Using a frozen ConvNeXt-Base, we compare three final-layer selection methods: \textbf{CovVol} (Algorithm~\ref{alg:covvol}), \textbf{Energy} (top-$k$ $\ell_2$ activation energy), and \textbf{Random} (50-seed mean$\pm$std). For pruning ratio $\rho$, we train a linear probe on $\lceil(1-\rho)C_{L-1}\rceil$ retained channels. Evaluation covers \textbf{CIFAR-100} (ID), \textbf{CIFAR-100-C}~\cite{hendrycks2019benchmarking} (OOD), \textbf{Stanford Dogs}~\cite{khosla2011novel}, and \textbf{Oxford-IIIT Pet} (cross-dataset).

\textbf{Finding 1: A broad compression plateau.} ConvNeXt-Base retains high accuracy at $90\%$ channel pruning. The CovVol results change from $88.23$ to $87.50$ on CIFAR-100, $91.75$ to $91.36$ on Dogs, $94.39$ to $93.46$ on Pets, and $62.77$ to $61.23$ on CIFAR-100-C, a decline of $0.39$--$1.54$ percentage points. The proximity of the selection curves over much of this range indicates substantial representation redundancy, consistent with H1. The width of this plateau depends on the backbone and task (Supplementary~C.6).

\textbf{Finding 2: Covariance-volume selection improves extreme compression.} Methods separate at extreme compression ($\rho \geq 95\%$). At $98\%$ pruning, CovVol outperforms Energy and Random on CIFAR-100 ($82.20\%$ vs. $80.17\%$ vs. $80.72\pm 0.46\%$) and CIFAR-100-C ($51.96$ vs. $48.98$ vs. $50.99\pm 0.63$). These tail-budget results demonstrate the practical value of selecting channels by response volume.

\textbf{Finding 3: Corruption degrades both reconstruction and response geometry.} Figure~\ref{fig:covvol_main}(c) shows foreground structure weakening and background content becoming more prominent with corruption, consistent with disrupted cancellation. Panel~(d) tracks accuracy, $\log\det\Sigma$, effective rank, $\mathrm{Tr}(\Sigma)$, and MMD~\cite{gretton2012kernel}. Accuracy and covariance summaries decline as feature-distribution discrepancy grows, with different response magnitudes across metrics. The reported clean effective rank of approximately $123$ is of the same order as the channel budgets near the onset of compression losses. These results connect the degradation of class reconstructions to changes in the feature-response geometry.

\subsection{Comparison with CAM-Based Attribution}
\label{sec:cam_comparison}

The selection experiments assess the utility of channel subsets. We now examine whether the spatial energy of a class reconstruction identifies image regions that influence the classifier's output. Our saliency maps use the per-pixel $\ell_2$ magnitude of $\hat{X}^{(c)}$. We compare official \textbf{Grad-CAM}~\cite{selvaraju2017grad}, \textbf{Grad-CAM++}~\cite{chattopadhay2018grad}, and \textbf{Score-CAM}~\cite{wang2020score} on ConvNeXt-Base across CUB-200 and Oxford-IIIT Pet.
\begin{table}[t]
\centering
\small
\setlength{\tabcolsep}{3pt}

\begin{minipage}{0.40\textwidth}
\caption{Faithfulness (Insertion/Deletion AUC). Higher Ins / lower Del is better.}
\label{tab:faithfulness}
\begin{tabular}{lcccc}
\toprule
 & \multicolumn{2}{c}{CUB-200} & \multicolumn{2}{c}{Pet} \\
\cmidrule(lr){2-3}\cmidrule(lr){4-5}
 & Ins$\uparrow$ & Del$\downarrow$ & Ins$\uparrow$ & Del$\downarrow$ \\
\midrule
G-CAM       & .563 & .489 & .746 & .641 \\
G-CAM++     & .628 & .409 & .698 & .654 \\
S-CAM       & .573 & .426 & .740 & .571 \\
\textbf{Ours} & \textbf{.762} & \textbf{.164} & \textbf{.817} & \textbf{.384} \\
\bottomrule
\end{tabular}
\end{minipage}\hfill
\begin{minipage}{0.56\textwidth}
\caption{Localization (mIoU/mAP/pAcc/Pointing). Best \textbf{bold}, second \underline{underlined}.}
\label{tab:localization}
\resizebox{\textwidth}{!}{%
\begin{tabular}{lcccccccc}
\toprule
 & \multicolumn{4}{c}{CUB-200} & \multicolumn{4}{c}{Pet} \\
\cmidrule(lr){2-5}\cmidrule(lr){6-9}
 & mIoU & mAP & pAcc & Point & mIoU & mAP & pAcc & Point \\
\midrule
G-CAM     & \underline{.401} & .617 & \underline{.648} & .614 & .294 & .422 & \underline{.653} & .379 \\
G-CAM++   & .345 & \underline{.664} & .625 & \underline{.705} & .245 & .400 & .635 & .347 \\
S-CAM     & \textbf{.576} & .647 & \textbf{.698} & .544 & \textbf{.470} & \underline{.482} & .619 & \underline{.409} \\
\textbf{Ours} & .319 & \textbf{.748} & .573 & \textbf{.888} & \underline{.428} & \textbf{.509} & \textbf{.698} & \textbf{.460} \\
\bottomrule
\end{tabular}
}
\end{minipage}
\end{table}

\textbf{Faithfulness} (Table~\ref{tab:faithfulness}). Our method leads all compared Insertion and Deletion AUC results. On CUB-200 it improves over the best baseline by approximately $13.4$ and $24.5$ percentage points, respectively; the corresponding Pet improvements are $7.1$ and $18.7$ points. \textbf{Localization} (Table~\ref{tab:localization}). It leads mAP and Pointing on both datasets and pixel accuracy on Pet, while Score-CAM leads mIoU. This contrast fits the interference account: energy concentrated on discriminative subregions can be useful for the decision without covering the full object mask. The comparisons demonstrate strong perturbation-based faithfulness alongside selective spatial localization (Supplementary~C.7).

\subsection{Causal Concept Localization via Channel Ablation}
\label{sec:causal_ablation}

\begin{figure}[!htbp]
\centering
\includegraphics[width=0.40\linewidth]{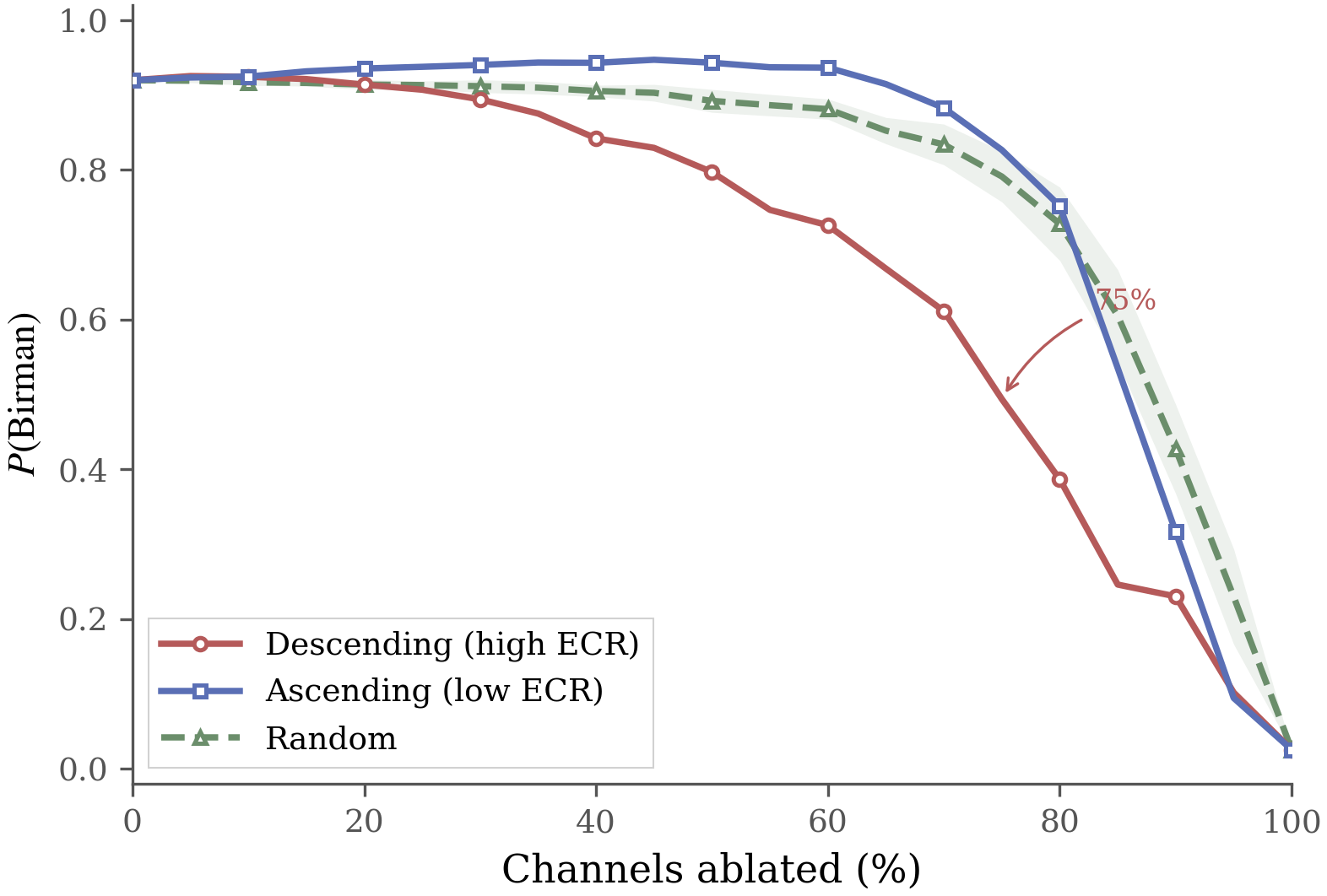}
\caption{Channel ablation: descending ECR, ascending ECR, and random removal.}
\label{fig:ecr_ablation}
\end{figure}
The CAM comparison evaluates image regions identified by the full class reconstruction. At the channel level, Network Dissection measures alignment between hidden-unit activations and annotated semantic concepts~\cite{Bau_2017_CVPR}. We next intervene on the channels themselves to test whether their regional inversion preferences predict their influence on classification. H1 allows such preferences to coexist with shared global scene content. For a reference $X_{\text{ref}}$ and region $\mathcal{R}$ (here, the tail of one Birman cat), we define the Energy Concentration Ratio $\text{ECR}_i$ as the proportion of squared $\ell_2$ inversion energy within $\mathcal{R}$: $\text{ECR}_i := \sum_{(u,v)\in\mathcal{R}} \|\tilde{V}_i(u,v)\|^2 / \sum_{(u,v)} \|\tilde{V}_i(u,v)\|^2$.

ECR ranks channels by the spatial concentration of their inversions. We compute the ranking from $X_{\mathrm{ref}}$ and apply it to every Birman test image using three orders: \textbf{Descending} (high-ECR first), \textbf{Ascending} (low-ECR first), and \textbf{Random} (50 seeds). Figure~\ref{fig:ecr_ablation} reports $P(\mathrm{Birman})$ under channel removal. At $70\%$ ablation the probabilities are $0.61$, $0.88$, and $0.83$, respectively. The maximum separation between descending and ascending orders exceeds $0.3$. High-ECR channels therefore have a greater causal influence on class predictions in this intervention, and the single-image ranking transfers across the class. At the most severe ablations the ascending curve also drops sharply, consistent with the need to retain a sufficiently diverse representation. These results connect global scene support with functionally meaningful regional preferences.

\subsection{Application to a Pure Vision Transformer}
\label{sec:vit_generalization}
\label{sec:pure_vit}

We finally examine whether the same interface-calibration design improves inversion in a native Transformer encoder. We apply LAC to a pure ViT-B/16~\cite{dosovitskiy2020vit} by placing channel-wise GroupNorm correctors at Transformer block boundaries and including a stem corrector in the backward cascade. These block boundaries take the role of CNN stage interfaces. We evaluate trained ViT-LAC, untrained ViT-LAC, and raw VJP on the same 185 Pets test images, using the same top-128 energy dimensions for reconstruction in all three conditions.

\begin{table}[t]
\centering
\small
\caption{Pure ViT-B/16 inversion on 185 Pets test images, using the same top-128 energy dimensions in every condition. Metrics are computed on unenhanced tensors. Brackets give image-bootstrap 95\% confidence intervals.}
\label{tab:vit_inversion}
\begin{tabular}{lcc}
\toprule
Condition & Centered input correlation & Output/input contrast\\
\midrule
Trained ViT-LAC & $0.217\ [0.194,0.239]$ & $0.081\ [0.077,0.087]$\\
Untrained ViT-LAC & $0.009\ [0.008,0.010]$ & $0.069\ [0.065,0.073]$\\
Raw VJP & $-0.038\ [-0.040,-0.035]$ & $0.003\ [0.003,0.003]$\\
\bottomrule
\end{tabular}
\end{table}

Trained ViT-LAC improves centered input correlation to $0.217$, compared with $0.009$ for untrained LAC and $-0.038$ for raw VJP (Table~\ref{tab:vit_inversion}). Their output/input contrast ratios are $0.081$, $0.069$, and $0.003$, respectively. The untrained operator produces appreciable contrast but little input alignment, showing that the trained adaptation improves alignment beyond the contrast introduced by normalization alone. All metrics use unenhanced output tensors; Appendix~\ref{app:pure_vit_details} gives their definitions and evaluation details.

\begin{figure}[t]
\centering
\begingroup
\setlength{\tabcolsep}{1pt}
\scriptsize
\begin{tabular}{@{}*{9}{c}@{}}
\shortstack{Input} & \shortstack{B0\\Conv2d\\proj} & \shortstack{B13\\LayerNorm\\norm2} & \shortstack{B27\\Linear\\proj} & \shortstack{B41\\LayerNorm\\norm1} & \shortstack{B54\\Linear\\fc1} & \shortstack{B68\\Attention\\attn} & \shortstack{B82\\Linear\\qkv} & \shortstack{B96\\Linear\\fc2}\\[3pt]
\multicolumn{9}{l}{\textbf{(a) Direct display}}\\[2pt]
\includegraphics[width=0.105\linewidth]{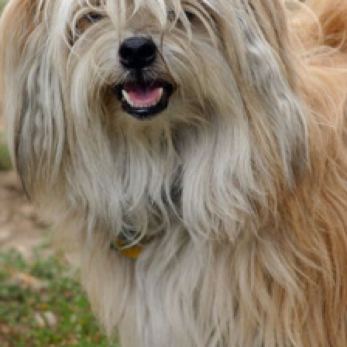} & \includegraphics[width=0.105\linewidth]{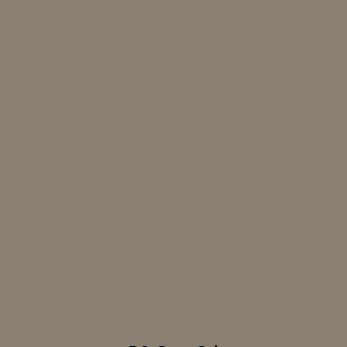} & \includegraphics[width=0.105\linewidth]{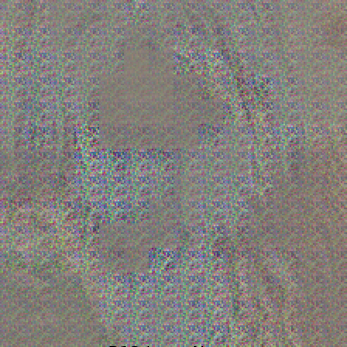} & \includegraphics[width=0.105\linewidth]{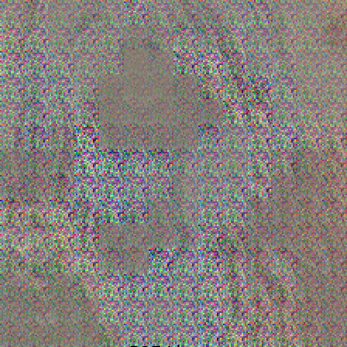} & \includegraphics[width=0.105\linewidth]{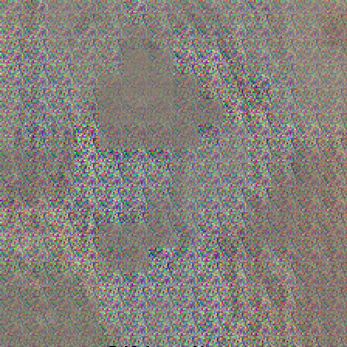} & \includegraphics[width=0.105\linewidth]{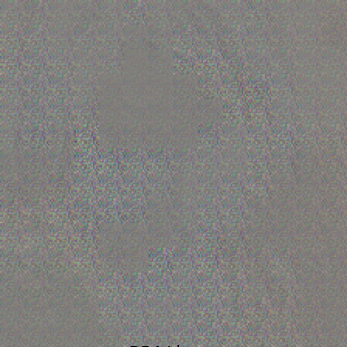} & \includegraphics[width=0.105\linewidth]{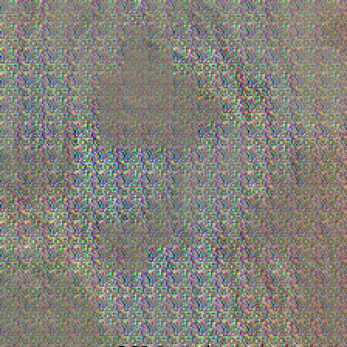} & \includegraphics[width=0.105\linewidth]{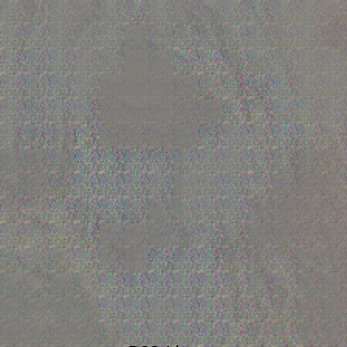} & \includegraphics[width=0.105\linewidth]{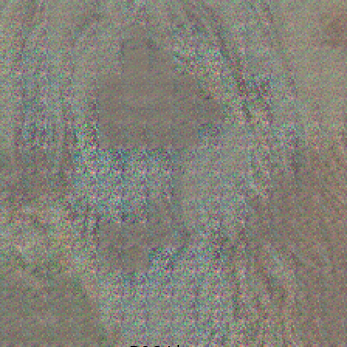}\\[2pt]
\includegraphics[width=0.105\linewidth]{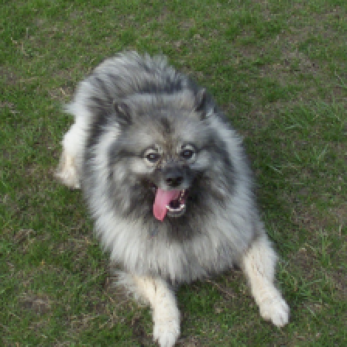} & \includegraphics[width=0.105\linewidth]{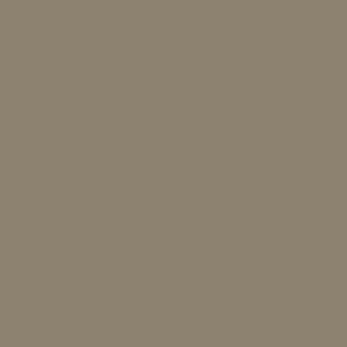} & \includegraphics[width=0.105\linewidth]{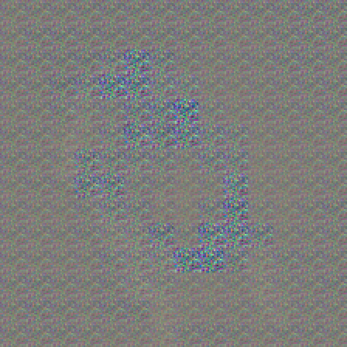} & \includegraphics[width=0.105\linewidth]{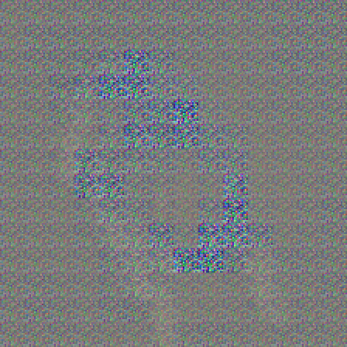} & \includegraphics[width=0.105\linewidth]{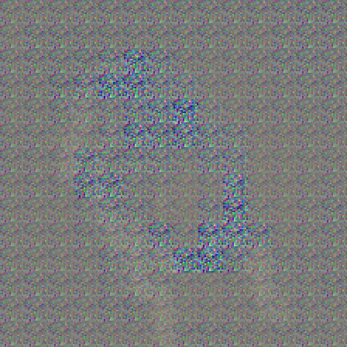} & \includegraphics[width=0.105\linewidth]{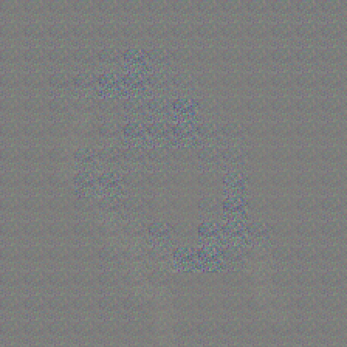} & \includegraphics[width=0.105\linewidth]{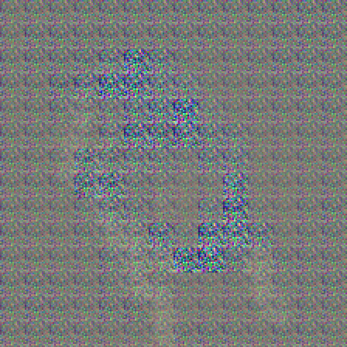} & \includegraphics[width=0.105\linewidth]{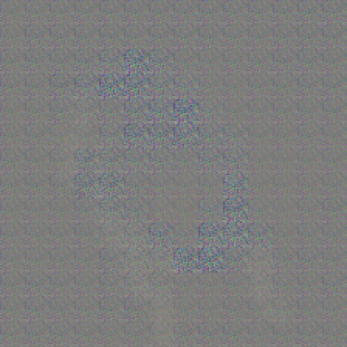} & \includegraphics[width=0.105\linewidth]{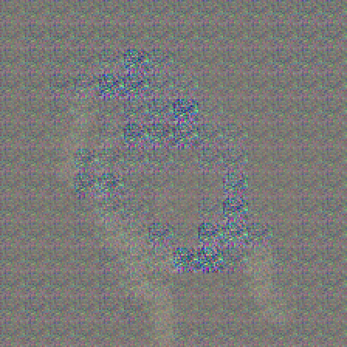}\\[5pt]
\multicolumn{9}{l}{\textbf{(b) Contrast-adjusted display}}\\[2pt]
\includegraphics[width=0.105\linewidth]{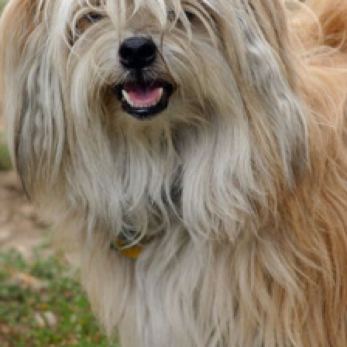} & \includegraphics[width=0.105\linewidth]{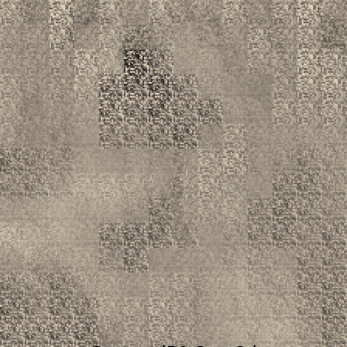} & \includegraphics[width=0.105\linewidth]{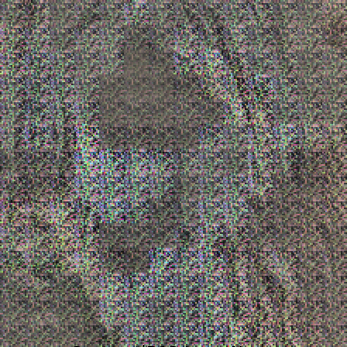} & \includegraphics[width=0.105\linewidth]{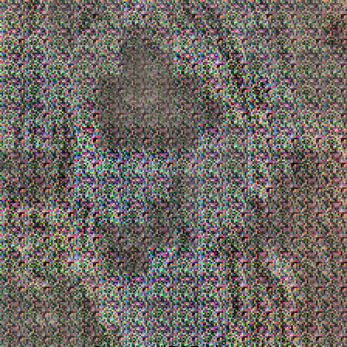} & \includegraphics[width=0.105\linewidth]{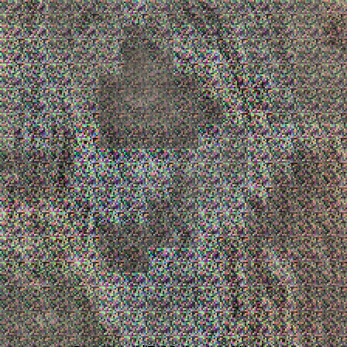} & \includegraphics[width=0.105\linewidth]{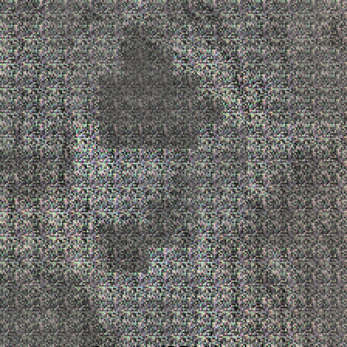} & \includegraphics[width=0.105\linewidth]{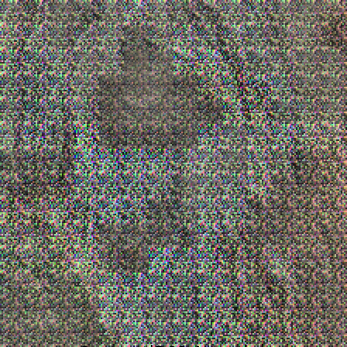} & \includegraphics[width=0.105\linewidth]{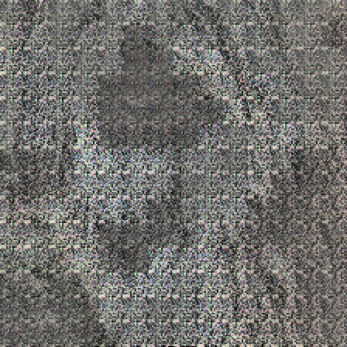} & \includegraphics[width=0.105\linewidth]{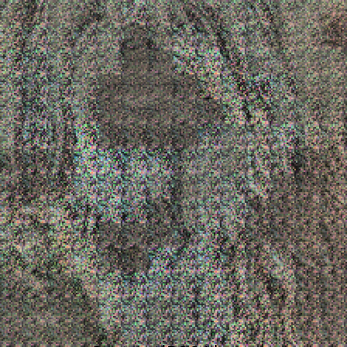}\\[2pt]
\includegraphics[width=0.105\linewidth]{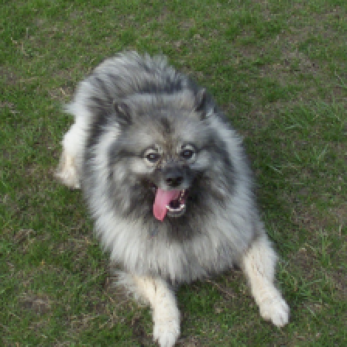} & \includegraphics[width=0.105\linewidth]{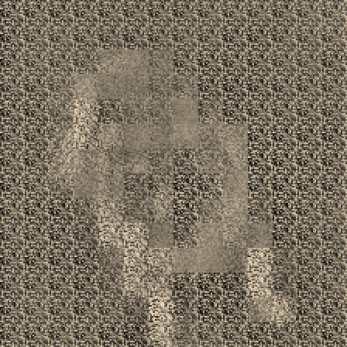} & \includegraphics[width=0.105\linewidth]{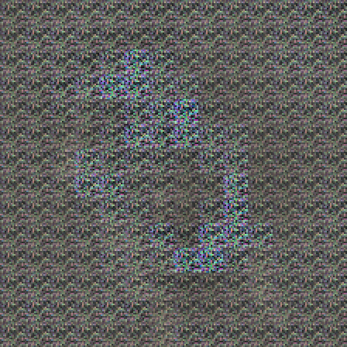} & \includegraphics[width=0.105\linewidth]{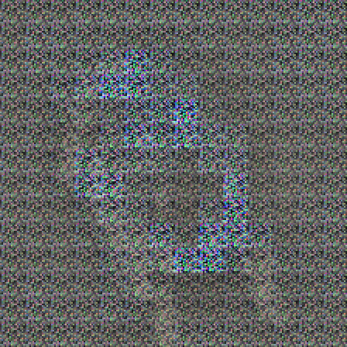} & \includegraphics[width=0.105\linewidth]{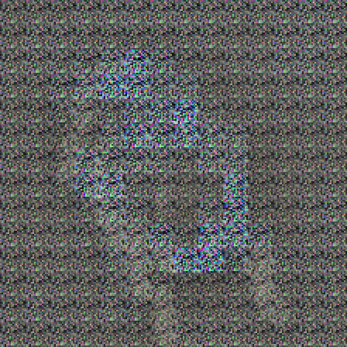} & \includegraphics[width=0.105\linewidth]{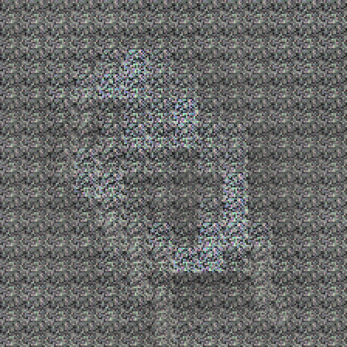} & \includegraphics[width=0.105\linewidth]{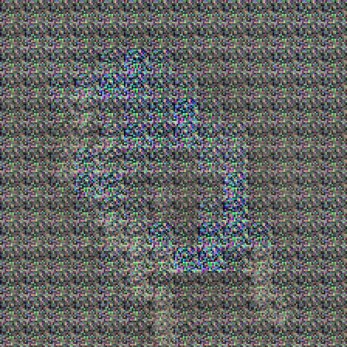} & \includegraphics[width=0.105\linewidth]{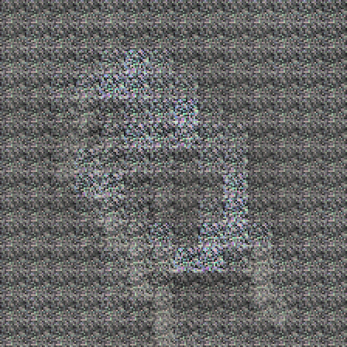} & \includegraphics[width=0.105\linewidth]{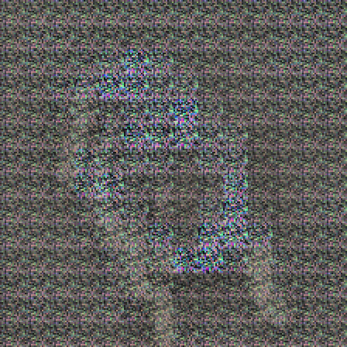}\\[2pt]
\end{tabular}
\endgroup
\caption{Pure-ViT inversions for two Pets images. (a) Direct display. (b) The same inversions after contrast adjustment for visualization. Column headings identify the selected readout modules; LACs act at Transformer block boundaries and the stem. Quantitative metrics in Table~\ref{tab:vit_inversion} use unenhanced output tensors.}
\label{fig:pure_vit_inversion}
\end{figure}

Figure~\ref{fig:pure_vit_inversion} shows inversions from selected readout modules along the ViT pathway. Low-contrast structure in the direct outputs becomes easier to inspect after display-only contrast adjustment, revealing object contours and coarse spatial layout at multiple readouts. These visualizations complement the quantitative evidence for input alignment.

We also apply the signed-reconstruction diagnostic used for CNNs. In this ViT-B/16/Pets setting, the full signed reconstruction does not show the foreground-energy gain over both sign-restricted components observed in the CNN experiments (Appendix~\ref{app:pure_vit_details}).

\FloatBarrier
\section{Conclusion}
\label{sec:conclusion}
We develop an input-conditioned adjoint inversion framework to examine how CNN channels retain scene content and combine it along class directions. The resulting inversions reveal a holographic organization in which shared scene content coexists with channel-specific residuals. Their signed class-weighted combinations expose foreground concentration through destructive interference. Per-image spectra and reconstruction controls characterize this organization, while channel interventions connect regional inversion preferences to classifier behavior. CovVol translates the geometric motivation into an exact feature-response volume objective and effective channel selection under compression and corruption. A pure-ViT adaptation further extends the inversion method and improves input alignment. The framework thus connects global scene support in individual channel inversions to selective class reconstructions and useful representation compression.

\clearpage
{\small
\bibliographystyle{plain}
\bibliography{reference}
}
\clearpage
\appendix
\section{Additional Method Details}
\label{app:method-details}

\subsection{Zero-Insertion Artifacts and Directional Sensitivity}
\label{app:two-crises}

Section~\ref{sec:crises} of the main text identifies two algebraic pathologies of the raw VJP. We describe the zero-insertion operation and prove the conditional directional-sensitivity bound in Eq.~\eqref{eq:angle-collapse}.

\subsubsection*{Zero insertion as a discrete Dirac comb.}

Let a stage transition apply a stride-$s$ convolution with kernel $W$ to a feature map $h_{l-1}$, producing $h_l$. In the forward pass this acts as decimation: only spatial positions on the stride lattice are retained. In the backward pass, the chain rule applied to $h_l = W \circledast_{(s)} h_{l-1}$ yields
\begin{equation}
\label{eq:adjoint-stride}
\nabla_{h_{l-1}}\mathcal{J} \;=\; W^{T} \circledast \,\uparrow_{s}\!(\nabla_{h_l}\mathcal{J}),
\end{equation}
where $\uparrow_s$ is the zero insertion upsampling operator and $\circledast$ is standard convolution on the upsampled grid. The operator $\uparrow_s$ produces a sparse signal whose support is the integer lattice $s\mathbb{Z}\times s\mathbb{Z}$,
\begin{equation}
[\uparrow_s u](r,t) \;=\; \sum_{p,q} u(p,q)\,\delta(r-sp,\;t-sq).
\end{equation}
In signal-processing terms, zero insertion places the coarse signal on a finer sampling lattice and produces periodic spectral images. The following transposed convolution determines how these images are filtered; its interpolation can leave periodic spatial artifacts. This identifies downsampling boundaries as natural locations for calibration, while the quality of raw inversion and the calibrated cascade is evaluated in Section~\ref{sec:exp_comparison}.

\subsubsection*{Magnitude and spatial direction.}

For any pixel space VJP vector $v\in\mathbb{R}^D$ with $v\neq 0$, write
\begin{equation}
v \;=\; \|v\|\,\cdot\,\frac{v}{\|v\|}.
\end{equation}
Magnitude and direction respond differently to perturbations. If the spatial component of interest becomes small relative to an additive perturbation, the normalized direction may change substantially. The following calculation quantifies this directional sensitivity under relative perturbation bounds.

\subsubsection*{Proof of the angular collapse inequality~\eqref{eq:angle-collapse}.}

Decompose $\epsilon$ along $v$ as $\epsilon = \epsilon_\parallel + \epsilon_\perp$ with $\epsilon_\parallel = \alpha v$ for some scalar $\alpha\in\mathbb{R}$ and $\langle \epsilon_\perp, v\rangle = 0$. Then
\begin{equation}
v + \epsilon \;=\; (1+\alpha)\,v + \epsilon_\perp,\qquad
\|v+\epsilon\|^2 \;=\; (1+\alpha)^2\|v\|^2 + \|\epsilon_\perp\|^2.
\end{equation}
Define $\hat{v} := v/\|v\|$ and $\hat{w} := (v+\epsilon)/\|v+\epsilon\|$. Their inner product is
\begin{equation}
\langle \hat{v},\hat{w}\rangle
\;=\; \frac{(1+\alpha)\|v\|^2}{\|v\|\sqrt{(1+\alpha)^2\|v\|^2 + \|\epsilon_\perp\|^2}}
\;=\; \frac{1+\alpha}{\sqrt{(1+\alpha)^2 + \rho^2}},
\quad
\rho := \frac{\|\epsilon_\perp\|}{\|v\|}.
\end{equation}
Under the hypotheses $\|\epsilon_\parallel\|\leq K\|v\|$ (so $|\alpha|\leq K$ and $1+\alpha\leq 1+K$) and $\|\epsilon_\perp\|\geq\|v\|$ (so $\rho\geq 1$), the right hand side is bounded above by
\begin{equation}
\langle \hat{v},\hat{w}\rangle
\;\leq\; \frac{1+K}{\sqrt{(1+K)^2 + 1}},
\end{equation}
where we used the monotonicity of $t\mapsto t/\sqrt{t^2+\rho^2}$ in $t$ and the lower bound $\rho^2\geq 1$. Since $\hat{v}$ and $\hat{w}$ are unit vectors, $\|\hat{w}-\hat{v}\|^2 = 2 - 2\langle\hat{v},\hat{w}\rangle$, so
\begin{equation}
\|\hat{w}-\hat{v}\|^2
\;\geq\; 2 - \frac{2(1+K)}{\sqrt{(1+K)^2+1}}
\;=\; \frac{2\bigl(\sqrt{(1+K)^2+1} - (1+K)\bigr)}{\sqrt{(1+K)^2+1}}.
\end{equation}
Multiplying numerator and denominator by $\sqrt{(1+K)^2+1}+(1+K)$ yields
\begin{equation}
\|\hat{w}-\hat{v}\|^2
\;\geq\; \frac{2}{\bigl(\sqrt{(1+K)^2+1}\bigr)\bigl(\sqrt{(1+K)^2+1}+(1+K)\bigr)}
\;\geq\; \frac{1}{(1+K)^2+1},
\end{equation}
where the second inequality uses $\sqrt{(1+K)^2+1}+(1+K)\leq 2\sqrt{(1+K)^2+1}$. Taking the square root,
\begin{equation}
\|\hat{w}-\hat{v}\|
\;\geq\; \frac{1}{\sqrt{(1+K)^2+1}}
\;=\; c \;>\; 0,
\end{equation}
which proves~\eqref{eq:angle-collapse}. This bound motivates separating magnitude from spatial structure in the reconstruction design.

\subsection{Existence and Geometry of the Simplex Measure}
\label{app:simplex}

Section~\ref{sec:forward} defines the per-layer simplex measure $E_{l,c}$ via $Z_{l,c}=\mathrm{GAP}(|h_{l,c}|)$ and $E_{l,c}=Z_{l,c}/\sum_{c'}Z_{l,c'}$. We note three properties.

\textbf{Existence.} The denominator is positive exactly when at least one channel has a nonzero activation. We define the simplex measure on this domain. An all-zero feature tensor has no nonzero channel energy to allocate.

\textbf{Simplex membership.} By construction, $E_{l,c}\geq 0$ and $\sum_c E_{l,c}=1$. The vector $(E_{l,1},\ldots,E_{l,C_l})$ resides on the standard $(C_l-1)$-dimensional probability simplex
\begin{equation}
\Delta^{C_l-1} \;=\; \Bigl\{p\in\mathbb{R}_{\geq 0}^{C_l}\;:\;\textstyle\sum_c p_c = 1\Bigr\},
\end{equation}
and represents relative activation magnitude across channels. This allocation provides nonnegative synthesis weights with a fixed sum.

\textbf{Signed responses in the backward seed.} Taking absolute values before pooling makes the forward weights nonnegative for both nonnegative and signed feature maps, including those produced by GELU or Swish. The original signed activation enters $\mathrm{seed}_{l,c}$ in Eq.~\eqref{eq:vjp-seed}; together with the input-conditioned Jacobian, it determines the backward inversion.

\subsection{Local Calibration: Geometry and Placement}
\label{app:lac-geometry}

Section~\ref{sec:lac} introduces the LAC and the GroupNorm strip~\eqref{eq:gn-strip}. We expand on three aspects: deployment topology, the three-step geometric factorisation, and the cascade depth seen by each stage.

\subsubsection*{Stage and stem interfaces.}

For the CNN encoders, LAC modules are placed at the stage boundaries and stem in~\eqref{eq:markov-decomp}. Resolution-changing operations motivate this placement; the cascade explicitly lists each corrected interface and retains the frozen operations within each stage. Intra-stage operations retain their frozen adjoints. This placement concentrates calibration at the changes in spatial resolution while keeping the number of trainable parameters small.

\subsubsection*{The three geometric actions of GroupNorm.}

For a spatial plane $v\in\mathbb R^n$, write $\mu=\mathbf1^Tv/n$, $P=I-\mathbf1\mathbf1^T/n$, and $\sigma^2=\|Pv\|_2^2/n$. The operations are
\begin{equation}
v\longmapsto Pv\longmapsto\frac{Pv}{\sqrt{\sigma^2+\varepsilon}}
\longmapsto\gamma_c\frac{Pv}{\sqrt{\sigma^2+\varepsilon}}+\beta_c\mathbf1.
\end{equation}
\emph{Centering} removes the spatially uniform component. \emph{Variance normalization} controls scale while preserving the direction of the centered plane; in the idealized $\varepsilon=0$, $\sigma>0$ case, it projects onto the radius-$\sqrt n$ sphere in the zero-mean hyperplane. \emph{Affine calibration} applies a uniform contrast scale and brightness offset. The scale can have either sign. These actions introduce no learned spatial kernel: their local spatial-difference identity is given in Corollary~\ref{cor:local-spatial-action}. Intermediate corrected signals are subsequently mixed by the frozen Jacobians in the cascade.

\subsubsection*{Cascade depth seen by each source stage.}

Suppose the encoder has $L$ feature stages plus a stem. A feature at stage $l$ traverses, during the backward cascade~\eqref{eq:lac-cascade}, exactly the LAC modules
\begin{equation}
\Psi_{\mathrm{LAC}_l},\;\Psi_{\mathrm{LAC}_{l-1}},\;\ldots,\;\Psi_{\mathrm{LAC}_0},\;\Psi_{\mathrm{LAC}_{\mathrm{stem}}},
\end{equation}
which is $l+2$ LAC modules in total. The deepest stage $L-1$ exercises every LAC in the system; any shallower stage exercises a strict subset. This nested structure is what enables the ``train once, infer at every depth'' property of the framework, and is what underlies the shared-module result of Proposition~\ref{prop:path-invariance}.

\subsection{Training Strategy and Inference Spectrum: Detailed Analysis}
\label{app:training}

Section~\ref{sec:synthesis} trains the cascade with the deepest-stage objective $\mathcal L=\|X-\widehat X_{L-1}\|_1$.

\subsubsection*{Optimization protocol.}
For each dataset--backbone classifier, the encoder remains frozen and one LAC checkpoint is trained on the official training split. Adam updates only the GN affine scales and offsets for $10{,}000$ steps, with learning rate $10^{-3}$ and gradient-norm clipping at $0.01$. Each step draws eight source channels with replacement using the batch-mean activation-energy distribution. Batch sizes are 20 for ImageNet and 24 otherwise. Weight decay is $10^{-5}$ on ImageNet/CIFAR-100 and $10^{-4}$ on CUB/Pets/Dogs. Inputs use ImageNet normalization for ImageNet and channel-wise mean and standard deviation $0.5$ for the other datasets. The trained checkpoint is reused across source stages and inversion analyses of that model.

\subsubsection*{Why train on the deepest stage only.}
The deepest path traverses every LAC module, providing a single reconstruction objective for all shared parameters. Training on that path avoids introducing separate loss weights for reconstructions from different stages. The pixelwise $L_1$ objective directly penalizes reconstruction error; the quality of the resulting reconstructions is reported in Section~\ref{sec:exp_comparison}.

\subsubsection*{Inference across source stages.}
A source at layer $l$ uses $\Psi_{\mathrm{LAC}_l},\ldots,\Psi_{\mathrm{LAC}_0},\Psi_{\mathrm{LAC}_{\mathrm{stem}}}$, a subset of the deepest path's modules. Reusing them produces the layer-wise inversion family $\{\widehat X_l\}_{l=0}^{L-1}$ without retraining. In the displayed reconstructions, depth changes the frequency content and semantic emphasis while the global scene remains visible.

\subsection{Shared Adjoint Bridges and Cross-Layer Reuse}
\label{app:shared-bridges}
The cascade shares its lower-level modules across source layers. Per-plane centering gives zero spatial mean, and variance normalization gives second moment $\sigma^2/(\sigma^2+\varepsilon)$ before the affine transformation. Each shared affine transformation has the same channel-wise parameters for all source layers. Its spatial action is a uniform scale and shift, while its downstream computation follows the same frozen Jacobians. Proposition~\ref{prop:path-invariance} formalizes these properties.

For any differentiable head acting on the deepest feature map, a scalar readout $\mathcal Q$ supplies the seed $\partial\mathcal Q/\partial h_{L-1}$ to the trained cascade. This construction reuses the encoder-side calibration. Appendix~\ref{app:hybrid_attention} demonstrates the procedure with three attention blocks on a frozen CNN, using both layer-wise and token-wise energy queries.

\section{Theoretical Statements and Proofs}
\label{app:theory}

This appendix establishes the algebraic properties of the adjoint and LAC
cascade. We first derive the final affine decomposition, the activation-energy
gradient identity, and the exact GroupNorm statistics. We then prove raw-adjoint
support containment, relate the cascade to normalised feature-visualisation
directions, and describe the spatial structure and shared computation of LAC.
All encoder Jacobians are evaluated at the fixed input $X$; the gradient
identities are stated at points where these Jacobians exist.

\subsection{Proposition B.1: Final Affine Decomposition}
\label{app:lem-training-dynamics}

The source channel index $c$ and the output colour-plane index $r$ play
different roles. A channel seed at layer $l$ produces a three-plane signal
$f_{l,c}\in\mathbb{R}^{3\times H_{\rm in}\times W_{\rm in}}$ immediately before the
final, stem-level LAC:
\begin{equation}\label{eq:f-def}
 f_{l,c}:=
 \bigl[J_{\mathrm{stem}}^T\circ\Psi_{\mathrm{LAC}_0}\circ J_0^T
 \circ\cdots\circ\Psi_{\mathrm{LAC}_l}\circ J_l^T\bigr]
 (\mathrm{seed}_{l,c}).
\end{equation}
Write $f_{l,c,r}\in\mathbb{R}^{n}$ for its flattened colour plane,
where $n=H_{\rm in}W_{\rm in}$ and $r\in\{1,2,3\}$. The final affine parameters
$(\gamma_r,\beta_r)$ are shared across source channels and layers.
Let $P:=I_n-\mathbf{1}\mathbf{1}^T/n$, and use the implementation
stabiliser $\varepsilon>0$ throughout.

\begin{proposition}[B.1: Final affine decomposition]
\label{lem:training-dynamics}
\label{prop:final-affine}
Define
\begin{equation}\label{eq:dcas-def}
 \begin{aligned}
 t_{l,c,r}&:=\|P f_{l,c,r}\|_2,\qquad
 d_{l,c,r}^{\mathrm{cas}}:=
 \begin{cases}P f_{l,c,r}/t_{l,c,r},&t_{l,c,r}>0,\\
 \mathbf{0},&t_{l,c,r}=0,
 \end{cases}
 \\
 a_{l,c,r}&:=\frac{\gamma_r t_{l,c,r}}
 {\sqrt{t_{l,c,r}^2/n+\varepsilon}}.
 \end{aligned}
\end{equation}
For arbitrary real affine parameters, the final LAC output satisfies
\begin{equation}\label{eq:vtilde-exact}
 \widetilde V_{l,c,r}
 =a_{l,c,r}d_{l,c,r}^{\mathrm{cas}}+\beta_r\mathbf{1}.
\end{equation}
If $E_{l,c}\geq0$ and $\sum_cE_{l,c}=1$, the layer reconstruction is
\begin{equation}\label{eq:recon-bias-split}
 \widehat X_{l,r}
 =\sum_c E_{l,c}a_{l,c,r}d_{l,c,r}^{\mathrm{cas}}
 +\beta_r\mathbf{1}.
\end{equation}
\end{proposition}

\begin{proof}
Per-plane GroupNorm gives
$\widetilde V_{l,c,r}=\gamma_r P f_{l,c,r}/
\sqrt{\|P f_{l,c,r}\|_2^2/n+\varepsilon}+\beta_r\mathbf{1}$.
For $t_{l,c,r}>0$, substituting~\eqref{eq:dcas-def} proves
\eqref{eq:vtilde-exact}. For $t_{l,c,r}=0$, both its centred term
and $a_{l,c,r}d_{l,c,r}^{\mathrm{cas}}$ vanish. Summing with the
simplex weights proves~\eqref{eq:recon-bias-split}.
\end{proof}

This decomposition holds for the learned parameters at any training step.
The signed coefficient $a_{l,c,r}$ contains the final affine scale and the
variance stabilisation factor; the bias remains explicit.

\subsection{Proposition B.2: VJP as Activation Energy Gradient}
\label{app:prop-vjp-energy}

\begin{proposition}[B.2: VJP as activation energy gradient]
\label{prop:vjp-energy}
For an encoder feature tensor
$h_l(X)\in\mathbb{R}^{C_l\times H_l\times W_l}$, define
\begin{equation}\label{eq:energy-def}
 \Phi_{l,c}(X):=\tfrac12\|h_{l,c}(X)\|_2^2,
\end{equation}
and the channel-selective vector--Jacobian product
\begin{equation}\label{eq:vjp-def}
 \mathrm{VJP}_{l,c}
 :=\left(\frac{\partial h_l}{\partial X}\right)^T
 (e_c\odot h_l).
\end{equation}
Then
\begin{equation}\label{eq:vjp-energy-app}
 \mathrm{VJP}_{l,c}=\nabla_X\Phi_{l,c}(X).
\end{equation}
\end{proposition}

\begin{proof}
Writing $J_{l,c}=\partial h_{l,c}/\partial X$, the selected seed
gives $\mathrm{VJP}_{l,c}=J_{l,c}^Th_{l,c}$. The chain rule gives
\begin{equation}\label{eq:grad-energy}
 \nabla_X\Phi_{l,c}
 =\sum_{u,v}h_{l,c}(u,v)\nabla_X h_{l,c}(u,v)
 =J_{l,c}^Th_{l,c},
\end{equation}
which proves the identity.
\end{proof}

\begin{remark}[Activation functions]\label{rmk:vjp-scope}
The identity applies to differentiable encoders with signed or non-negative
activations. For piecewise differentiable networks it applies on each
differentiable region. At activation kinks the implemented VJP uses the
automatic-differentiation convention for the local derivatives.
\end{remark}

\subsection{Proposition B.3: Algebraic Structure of GroupNorm}
\label{app:prop-gn-structure}

\begin{proposition}[B.3: Algebraic structure of GroupNorm]
\label{prop:gn-structure}
For one flattened spatial plane $v\in\mathbb{R}^{n}$, let
\begin{equation}\label{eq:P-perp}
 P_{\perp\mathbf{1}}=I_n-\tfrac1n\mathbf{1}\mathbf{1}^T,
 \qquad
 \mu=\tfrac1n\mathbf{1}^Tv,
 \qquad
 \sigma^2=\tfrac1n\|P_{\perp\mathbf{1}}v\|_2^2.
\end{equation}
With one GroupNorm group per channel,
\begin{equation}\label{eq:vhat}
 \hat v=\frac{P_{\perp\mathbf{1}}v}{\sqrt{\sigma^2+\varepsilon}},
\end{equation}
\begin{equation}\label{eq:gn-full}
 \mathrm{GN}(v)=\gamma\hat v+\beta\mathbf{1}.
\end{equation}
Its moments and centred magnitude satisfy
\begin{equation}\label{eq:gn-moments}
 \tfrac1n\mathbf{1}^T\hat v=0,\qquad
 \tfrac1n\|\hat v\|_2^2=\frac{\sigma^2}{\sigma^2+\varepsilon},
 \qquad
 \tfrac1n\mathbf{1}^T\mathrm{GN}(v)=\beta,
\end{equation}
\begin{equation}\label{eq:gn-magnitude}
 \|P_{\perp\mathbf{1}}\mathrm{GN}(v)\|_2
 =|\gamma|\sqrt n\frac{\sigma}{\sqrt{\sigma^2+\varepsilon}}
 \leq|\gamma|\sqrt n.
\end{equation}
For $\sigma>0$ and $\gamma\neq0$, its centred unit direction is
$\operatorname{sgn}(\gamma)P_{\perp\mathbf{1}}v/
\|P_{\perp\mathbf{1}}v\|_2$.
\end{proposition}

\begin{proof}
Since $\mathbf{1}^TP_{\perp\mathbf{1}}=0$ and
$\|P_{\perp\mathbf{1}}v\|_2^2=n\sigma^2$, substituting
\eqref{eq:vhat} into~\eqref{eq:gn-full} proves the moment and
magnitude identities. Projection removes the affine bias, leaving a scalar
multiple of $P_{\perp\mathbf{1}}v$, which gives the direction statement.
\end{proof}

In the idealised case $\varepsilon=0$ with $\sigma>0$, the centred
magnitude is exactly $|\gamma|\sqrt n$. With $\varepsilon>0$, the
formulas also cover spatially constant inputs, whose output is
$\beta\mathbf{1}$.

\subsection{B.4: Proof of Raw-Adjoint Spatial Fidelity}
\label{app:proof-spatial-fidelity}

For each source channel, define the effective field using all nonzero
activation entries:
\begin{equation}\label{eq:ef-general-app}
 \mathrm{EF}_{l,c}(X)
 :=\bigcup_{(u,v):\,h_{l,c}(u,v)\neq0}
 \operatorname{supp}_{\mathrm{sp}}\bigl(\nabla_Xh_{l,c}(u,v)\bigr).
\end{equation}
Here $\operatorname{supp}_{\mathrm{sp}}$ is the set of spatial pixels
where at least one input colour component is nonzero. For ReLU outputs,
the condition $h_{l,c}(u,v)\neq0$ is equivalent to
$h_{l,c}(u,v)>0$.

\begin{proof}[Proof of Theorem~\ref{thm:spatial-fidelity}]
By Proposition~\ref{prop:vjp-energy},
\begin{equation}\label{eq:vjp-expand}
 \mathrm{VJP}_{l,c}
 =\sum_{u,v}h_{l,c}(u,v)\nabla_Xh_{l,c}(u,v).
\end{equation}
If $h_{l,c}(u,v)=0$, its summand vanishes. Every other summand
has spatial support contained in~\eqref{eq:ef-general-app}.
The support of their sum is contained in the union of their supports,
proving $\operatorname{supp}_{\mathrm{sp}}(\mathrm{VJP}_{l,c})
\subseteq\mathrm{EF}_{l,c}(X)$. The inclusion allows cancellation
between contributions from different activation positions.
\end{proof}

\subsection{Theorem B.5: Feature-Visualisation Direction Decomposition}
\label{app:thm-fv-equivalence}

The raw energy gradient and the LAC cascade define two directions. We
compare them per output plane, matching the final GroupNorm operation.
Let $g_{l,c,r}:=[\nabla_X\Phi_{l,c}]_r$ and define
\begin{equation}\label{eq:fv-ideal-def}
 \mathrm{FV}_{l,c,r}(X):=
 \begin{cases}
 P g_{l,c,r}/\|P g_{l,c,r}\|_2,&\|P g_{l,c,r}\|_2>0,\\
 \mathbf{0},&\|P g_{l,c,r}\|_2=0.
 \end{cases}
\end{equation}
This is the centred, normalised energy-gradient direction, with zero
assigned to a vanishing centred gradient. Define the direction deviation
and its scaled magnitude by
\begin{equation}\label{eq:eta-def}
 \eta_{l,c,r}:=d_{l,c,r}^{\mathrm{cas}}-\mathrm{FV}_{l,c,r}(X),
\end{equation}
\begin{equation}\label{eq:eps-dc-def}
 \epsilon_{l,c,r}^{\mathrm{dir}}:=|a_{l,c,r}|\|\eta_{l,c,r}\|_2,
 \qquad \widetilde E_{l,c,r}:=E_{l,c}a_{l,c,r}.
\end{equation}

\begin{theorem}[B.5: Direction decomposition and residual bound]
\label{thm:fv-equivalence}
With the notation of Proposition~\ref{prop:final-affine},
\begin{equation}\label{eq:fv-equiv-main}
 \widehat X_{l,r}
 =\sum_{c=1}^{C_l}\widetilde E_{l,c,r}\mathrm{FV}_{l,c,r}(X)
 +R_{l,r},
\end{equation}
where
\begin{equation}\label{eq:xhat-decomposed}
 R_{l,r}:=\sum_c E_{l,c}a_{l,c,r}\eta_{l,c,r}
 +\beta_r\mathbf{1},
\end{equation}
and
\begin{equation}\label{eq:residual-bound}
 \|R_{l,r}\|_\infty
 \leq |\beta_r|+\sum_cE_{l,c}\epsilon_{l,c,r}^{\mathrm{dir}}.
\end{equation}
Consequently, the full RGB residual satisfies
$\|R_l\|_\infty\leq\max_r
(|\beta_r|+\sum_cE_{l,c}\epsilon_{l,c,r}^{\mathrm{dir}})$.
\end{theorem}

\begin{proof}
Substitute $d_{l,c,r}^{\mathrm{cas}}=
\mathrm{FV}_{l,c,r}+\eta_{l,c,r}$ into
\eqref{eq:recon-bias-split}. This gives
\eqref{eq:fv-equiv-main} and~\eqref{eq:xhat-decomposed} exactly.
The triangle inequality and $\|\eta\|_\infty\leq\|\eta\|_2$ give
\[
 \|R_{l,r}\|_\infty
 \leq|\beta_r|+
 \sum_cE_{l,c}|a_{l,c,r}|\|\eta_{l,c,r}\|_2.
\]
Taking the maximum over colour planes gives the RGB bound.
\end{proof}

The direction residual records the combined effect of all intermediate
LAC operations before the final normalisation. Since both directions in
\eqref{eq:eta-def} have norm at most one,
$\epsilon_{l,c,r}^{\mathrm{dir}}\leq2|a_{l,c,r}|$.
The decomposition retains the actual cascade signal~\eqref{eq:f-def};
each intermediate LAC performs its channelwise calibration before the next
Jacobian transpose is applied.

\subsection{Corollary B.6: Exact Expansion in Realised LAC Directions}
\label{app:cor-equiv-error}

\begin{corollary}[B.6: Realised feature-visualisation expansion]
\label{cor:equiv-error}
Define $\mathrm{FV}^{\mathrm{LAC}}_{l,c,r}(X)
:=d_{l,c,r}^{\mathrm{cas}}$. Then
\begin{equation}\label{eq:exact-decomp}
 \widehat X_{l,r}
 =\sum_{c=1}^{C_l}\widetilde E_{l,c,r}
 \mathrm{FV}^{\mathrm{LAC}}_{l,c,r}(X)+\beta_r\mathbf{1},
\end{equation}
and
\begin{equation}\label{eq:corollary-linf}
 \left\|\widehat X_{l,r}-\sum_c\widetilde E_{l,c,r}
 \mathrm{FV}^{\mathrm{LAC}}_{l,c,r}(X)\right\|_\infty
 =|\beta_r|.
\end{equation}
\end{corollary}

\begin{proof}
Equation~\eqref{eq:exact-decomp} is~\eqref{eq:recon-bias-split}
with the realised direction named explicitly. Subtracting its sum leaves
the spatial constant $\beta_r\mathbf{1}$, proving
\eqref{eq:corollary-linf}.
\end{proof}

Thus $\widehat X_{l,r}-\beta_r\mathbf{1}$ is exactly a weighted
expansion in realised LAC directions. Theorem~\ref{thm:fv-equivalence}
separately quantifies the difference between those directions and the raw
energy-gradient directions.

\subsection{B.7: Spatial Structure of the LAC Cascade}
\label{app:cascade-spatial-structure}

Let $D$ denote spatial finite differences between valid adjacent pixels.
For example, at an interior pixel,
\begin{equation}\label{eq:disc-grad-def}
 (Dv)_{i,j}=(v_{i+1,j}-v_{i,j},\;v_{i,j+1}-v_{i,j}).
\end{equation}
Differences crossing the image boundary are omitted, so $D\mathbf{1}=0$.
For a set $S$ of pixels, write
\begin{equation}\label{eq:dilation-def}
 S^{\oplus1}:=\{(i,j):\operatorname{dist}_1((i,j),S)\leq1\}.
\end{equation}
If $\operatorname{supp}_{\mathrm{sp}}(v)\subseteq S$, then
$\operatorname{supp}_{\mathrm{sp}}(Dv)\subseteq S^{\oplus1}$.

\begin{proof}[Proof of Corollary~\ref{cor:local-spatial-action}]
For a fixed spatial plane $v$, its mean, variance and affine parameters
are spatial scalars. Applying $D$ to~\eqref{eq:gn-full} therefore gives
\begin{equation}\label{eq:lac-spatial-identity-app}
 D\mathrm{GN}(v)
 =\frac{\gamma}{\sqrt{\sigma^2+\varepsilon}}Dv.
\end{equation}
The affine operation scales all spatial differences in that plane by the
same signed factor. Its spatial-gradient support is unchanged when
$\gamma\neq0$ and empty when $\gamma=0$.
\end{proof}

The final encoder adjoint also yields a direct support statement for the
cascade. Write $h_{\mathrm{stem}}(X)$ for the stem output and define
\begin{equation}\label{eq:stem-support-def}
 S_{\mathrm{stem}}(X):=
 \bigcup_q\operatorname{supp}_{\mathrm{sp}}
 (\nabla_Xh_{\mathrm{stem},q}(X)),
\end{equation}
where $q$ runs over all channel and spatial coordinates of that tensor.

\begin{lemma}[Stem-adjoint support of the cascade]
\label{lem:cascade-grad-containment}
For every source layer and channel,
\begin{equation}\label{eq:cascade-value-containment}
 \operatorname{supp}_{\mathrm{sp}}(f_{l,c})
 \subseteq S_{\mathrm{stem}}(X),
\end{equation}
and the layer reconstruction satisfies
\begin{equation}\label{eq:cascade-grad-containment}
 \operatorname{supp}_{\mathrm{sp}}(D\widehat X_l)
 \subseteq S_{\mathrm{stem}}(X)^{\oplus1}.
\end{equation}
\end{lemma}

\begin{proof}
The signal just before the final LAC has the form
$f_{l,c}=J_{\mathrm{stem}}^Tq_{l,c}$ for the tensor
$q_{l,c}$ supplied by the preceding cascade. Hence
$f_{l,c}=\sum_q(q_{l,c})_q\nabla_Xh_{\mathrm{stem},q}$,
which proves~\eqref{eq:cascade-value-containment}.
The dilation property gives
$\operatorname{supp}_{\mathrm{sp}}(Df_{l,c})
\subseteq S_{\mathrm{stem}}^{\oplus1}$.
Equation~\eqref{eq:lac-spatial-identity-app} preserves this containment
after the final LAC. Finally, the coefficients $E_{l,c}$ are spatial
scalars, so taking spatial differences commutes with the weighted sum
over $c$, proving~\eqref{eq:cascade-grad-containment}.
\end{proof}

These identities describe the division of labour in the decoder: the
frozen encoder Jacobians perform spatial mixing, while each learned LAC
uses only a scale and a bias per spatial plane.

\subsection{Proposition B.8: Shared LAC Structure Across Source Paths}
\label{app:prop-path-invariance}

At a bridge $b$, let $v_{b,j}^{(l)}\in\mathbb{R}^{n_b}$ be the
signal in bridge channel $j$ arriving from a source layer $l$ whose
backward path traverses that bridge. Its normalised signal and output are
\begin{equation}\label{eq:standardize-def}
 \hat v_{b,j}^{(l)}=
 \frac{v_{b,j}^{(l)}-\mu_{b,j}^{(l)}\mathbf{1}}
 {\sqrt{(\sigma_{b,j}^{(l)})^2+\varepsilon}},
\end{equation}
\begin{equation}\label{eq:affine-output-def}
 o_{b,j}^{(l)}=\gamma_j^{(b)}\hat v_{b,j}^{(l)}
 +\beta_j^{(b)}\mathbf{1},
 \qquad\theta_b=\{\gamma_j^{(b)},\beta_j^{(b)}\}_j.
\end{equation}

\begin{proposition}[B.8: Shared bridge structure]
\label{prop:path-invariance}
With one normalisation group per channel, the following properties hold
at a shared bridge:
\begin{enumerate}
\item The normalised spatial mean is zero and the spatial second moment is
$ (\sigma_{b,j}^{(l)})^2/((\sigma_{b,j}^{(l)})^2+\varepsilon)$.
For $\varepsilon=0$ and nonzero variance, the second moment is one.
\item For two source paths $A$ and $B$ using the same bridge,
\begin{equation}\label{eq:output-diff}
 o_{b,j}^{(A)}-o_{b,j}^{(B)}
 =\gamma_j^{(b)}(\hat v_{b,j}^{(A)}-\hat v_{b,j}^{(B)}).
\end{equation}
The learned affine acts uniformly over that plane.
\item The sequence of operations downstream of bridge $b$ is shared
across these source paths. In particular, the path from $\theta_b$
to the output has the same topology, although its signal values depend
on the source path.
\end{enumerate}
\end{proposition}

\begin{proof}
The first property follows from Proposition~\ref{prop:gn-structure}.
Subtracting the two instances of~\eqref{eq:affine-output-def} proves
the second. For the third, decompose a source path into its upstream
map to bridge $b$, the local bridge, and the downstream map from that
bridge to pixel space. The upstream map supplies $v_{b,j}^{(l)}$ and
contains no use of $\theta_b$. The downstream map uses the same encoder
adjoints and lower LAC modules for every such source path. For a scalar
loss $\mathcal L_l$ on a path, let
$g_{b,j}^{(l)}=\partial\mathcal L_l/\partial o_{b,j}^{(l)}$.
The local parameter derivatives are
\begin{equation}\label{eq:grad-gamma-formula}
 \frac{\partial\mathcal L_l}{\partial\gamma_j^{(b)}}
 =\langle g_{b,j}^{(l)},\hat v_{b,j}^{(l)}\rangle,
 \qquad
 \frac{\partial\mathcal L_l}{\partial\beta_j^{(b)}}
 =\langle g_{b,j}^{(l)},\mathbf{1}\rangle.
\end{equation}
The downstream computation producing $g_{b,j}^{(l)}$ has the same
graph for each source path, establishing the claim. When several seed
paths contribute to a loss, their parameter-gradient contributions sum.
\end{proof}

The deepest reconstruction path traverses all LAC modules, and shallower
paths reuse the corresponding downstream modules. A differentiable head
can similarly supply a seed at its encoder interface and use the same
trained cascade. The cross-layer reconstructions and attention-head
visualisations in Section~\ref{sec:experiments} and
Appendix~\ref{app:hybrid_attention} evaluate these uses of the shared architecture.

\section{Additional Experimental Details}
\label{app:experiments}

\subsection{Reconstruction Fidelity: Baselines, Resolution, and LAC Placement}
\label{app:exp_C1}

\paragraph{Qualitative reconstruction.}
Figure~\ref{fig:qualitative_comparison} complements the quantitative comparison in Table~\ref{tab:inversion_comparison}. Vanilla Gradient produces strong high-frequency artifacts, consistent with the adjoint effects discussed in Section~\ref{sec:crises}. Mahendran \& Vedaldi recover global silhouettes but exhibit color distortion and local texture artifacts in the illustrated examples. Robust Inversion shows similar reconstruction artifacts, while UpConvNet produces smooth reconstructions with reduced fine detail. LAC preserves the global scene layout together with local edges, colors, and textures. Its spatial mixing follows the frozen encoder's Jacobians, with channel-wise calibration along the adjoint path.

\paragraph{Baseline implementations and checkpoint reuse.}
The baselines in Table~\ref{tab:inversion_comparison} use their official public implementations. For Mahendran \& Vedaldi, we follow the natural-image prior formulation of~\cite{mahendran2015understanding}, which combines feature reconstruction with total-variation and $\ell_2$ regularization, using the regularization weights and optimization schedule of the original paper. Robust Inversion uses the publicly released adversarially robust encoder weights from~\cite{engstrom2019adversarial}. For UpConvNet, we train a separate feedforward decoder for each backbone--dataset pair following~\cite{dosovitskiy2016inverting}, until validation reconstruction loss plateaus. We train one LAC checkpoint for each dataset--backbone classifier and reuse it across source layers and inversion analyses of that model. The standard LAC objective is the deepest-stage $L_1$ pixel reconstruction loss~\eqref{eq:loss}, with no spatial smoothing or image prior. Only the LAC affine parameters are trained; they add no learned spatial filters to the frozen adjoint.

\paragraph{CIFAR-100 protocols and statistical units.}
Table~\ref{tab:cifar_protocol} separates three ResNet-18 comparisons. The enlarged protocol bilinearly maps $32\times32$ inputs to $224\times224$ and evaluates three training seeds on 2,000 test images per seed, with the trained encoder held fixed. The second setting evaluates those same models and outputs at $32\times32$. The native protocol uses $32\times32$ images directly, with three classifier/LAC/decoder seeds evaluated on all 10,000 test images per seed. Its training budgets are fixed separately for each method. The mean$\pm$SD entries summarize variation across seeds. Rescoring the same outputs isolates the effect of evaluation size, while the native experiment measures performance under its own training protocol.

\paragraph{Fixed-checkpoint interpolation controls.}
We additionally evaluate nearest-neighbor, bilinear, and bicubic interpolation with a fixed trained ResNet-18 checkpoint at $224\times224$ and $32\times32$ evaluation sizes. Define each paired difference as LAC minus UpConvNet on the same test image. Table~\ref{tab:cifar_interpolation} gives the 95\% confidence intervals for the bilinear control, using images as the statistical unit. These intervals describe the fixed-checkpoint comparison; the seed variability of the three-seed protocols appears separately in Table~\ref{tab:cifar_protocol}.

\begin{table}[t]
\centering\small
\caption{Fixed-checkpoint bilinear interpolation control on CIFAR-100/ResNet-18. Paired 95\% confidence intervals for LAC minus UpConvNet, with images as the statistical unit. Positive $\Delta\mathrm{SSIM}$ and negative $\Delta\mathrm{LPIPS}$ favor LAC. Metrics use raw units.}
\label{tab:cifar_interpolation}
\begin{tabular}{lcc}
\toprule
Evaluation size & $\Delta\mathrm{SSIM}$ & $\Delta\mathrm{LPIPS}$ \\
\midrule
$32\times32$ & $[0.229,\,0.239]$ & $[-0.172,\,-0.161]$ \\
$224\times224$ & $[-0.217,\,-0.206]$ & $[0.150,\,0.160]$ \\
\bottomrule
\end{tabular}
\end{table}

Across the three interpolation methods, $32\times32$ evaluation favors LAC, with $\Delta\mathrm{SSIM}$ from $+0.231$ to $+0.247$ and $\Delta\mathrm{LPIPS}$ from $-0.183$ to $-0.166$. At $224\times224$, SSIM favors UpConvNet, with differences from $-0.212$ to $-0.111$, while the LPIPS ranking depends on interpolation. Together with the three-seed results, these controls show that reconstruction comparisons on enlarged low-resolution images are sensitive to both interpolation and evaluation scale.

\paragraph{Placement at internal boundaries and the stem.}
Table~\ref{tab:lac_placement} compares internal-boundary-only correction with the stem-plus-boundary cascade on Pets/ResNet-50. Internal-boundary-only LAC gives SSIM $0.518$ and LPIPS $0.522$; adding the stem gives SSIM $0.482$ and LPIPS $0.516$. Both configurations substantially improve SSIM over the raw VJP value of $0.119$. The comparison identifies a metric-dependent trade-off: internal-boundary-only correction has higher SSIM, while stem-plus-boundary correction has lower LPIPS. It supports calibration at the resolution-changing boundaries targeted by the architecture.

\begin{table}[t]
\centering\small
\caption{LAC placement control on Pets/ResNet-50. Internal-boundary-only correction omits the stem corrector; stem-plus-boundary uses the full cascade. Values use raw metric units. Bold marks the better calibrated configuration for each metric.}
\label{tab:lac_placement}
\begin{tabular}{lcc}
\toprule
Configuration & SSIM $\uparrow$ & LPIPS $\downarrow$ \\
\midrule
Raw VJP & $0.119$ & -- \\
Internal-boundary-only LAC & $\mathbf{0.518}$ & $0.522$ \\
Stem-plus-boundary LAC & $0.482$ & $\mathbf{0.516}$ \\
\bottomrule
\end{tabular}
\end{table}

\subsection{Foreground Energy Paradox: Setup and Numerical Detail}
\label{app:exp_C2}

This appendix expands Section~\ref{sec:fg_paradox}. We evaluate on two fine-grained benchmarks providing pixel-accurate foreground masks: \textbf{CUB-200-2011}~\cite{wah2011caltech} (200 bird species, $11{,}788$ images) and \textbf{Oxford-IIIT Pet}~\cite{parkhi2012cats} (37 cat/dog breeds, $7{,}349$ images). We test four frozen ImageNet-pretrained encoders: \textbf{ResNet-18}, \textbf{ResNet-50}~\cite{he2016deep}, \textbf{DenseNet-121}~\cite{huang2017densely}, and \textbf{ConvNeXt-Base}~\cite{liu2022convnet}. For each encoder we freeze all convolutional parameters and train a single linear probe $w^{(c)}$ on the target dataset using standard supervised classification protocol, so that the classifier directions under study are exactly those used for classification. The class index $c$ in the inversion is taken to be the classifier's top-1 predicted label. For each test image we compute the three reconstructions $\hat{X}^{(c)}$, $\hat{X}^{(c)}_{+}$, $\hat{X}^{(c)}_{-}$, evaluate $\mathrm{FG}$ on each, and average across the full test split.

\paragraph{Metric details.}
The foreground energy proportion
\begin{equation}
\mathrm{FG}(\hat{X}) = \frac{\sum_{(i,j)\in M_{\mathrm{fg}}} \sum_{k=1}^3 \hat{X}_k(i,j)^2}{\sum_{(i,j)} \sum_{k=1}^3 \hat{X}_k(i,j)^2}
\end{equation}
sums over the three RGB channels in both numerator and denominator, with no mean-centering applied. The metric thus operates on the raw reconstruction as produced by the framework. A value substantially above the area ratio $|M_{\mathrm{fg}}|/(H_{\rm in} W_{\rm in})$ indicates that the reconstruction concentrates its energetic mass on the object beyond a uniform energy distribution. The foreground area ratios are approximately $0.40\text{--}0.55$, while Full reconstructions reach approximately $0.82$ for ConvNeXt-Base on CUB. The consistent increase from the two sign-restricted components to Full quantifies the foreground concentration produced by their combination.

\subsection{Scene Content and Class-Weighted Composition}
\label{app:exp_C3}

This appendix expands Section~\ref{sec:fg_paradox} through two empirical questions: what scene content remains visible in individual deep-channel inversions, and how do their signed class-weighted combinations concentrate on the foreground?

\paragraph{Global scene content across channels.}
Figure~\ref{fig:channel_holography}(a) visualizes $\tilde V_{L-1,i}(X)$ for channels drawn from the high-, mid-, and low-energy groups of $E_{L-1,c}$. Across these groups, the inspected inversions preserve the foreground object, contextual surround, and overall layout. Low-energy channels retain scene structure alongside high-energy channels. Their differences concern texture frequencies, colors, edge polarities, and the relative energy assigned to image regions. Thus global scene content coexists with channel-specific regional preferences, as further examined by the interventions in Section~\ref{sec:causal_ablation}. This pattern occurs across the four evaluated CNN backbones.

\paragraph{Composition across stages and classifier signs.}
Figure~\ref{fig:channel_holography}(b) juxtaposes stage-wise reconstructions $\hat X_l$ ($l\in\{0,1,2,3\}$), the class-directional reconstruction $\hat X^{(c)}$, and its positive- and negative-weight components $\hat X^{(c)}_+$ and $\hat X^{(c)}_-$. Global layout remains visible across stages. A typical $\hat X_0$ contains fine detail, whereas $\hat X_3$ renders the scene with lower-frequency structure and attenuation of some non-discriminative regions.

Both sign-restricted components retain foreground and background content. Their sum $\hat X^{(c)}_++\hat X^{(c)}_-=\hat X^{(c)}$ concentrates more strongly on the foreground, matching the energy measurements $\mathrm{FG}(\hat X^{(c)}_+)\approx\mathrm{FG}(\hat X^{(c)}_-)<\mathrm{FG}(\hat X^{(c)})$. The exact additive relation exposes cancellation between the components. The shared-scene decomposition in Supplementary~\ref{app:exp_C4} explains how this combination suppresses common content while retaining discriminative structure.

\paragraph{Additional examples.}
Figures~\ref{fig:dogs_signed_composition} and~\ref{fig:dogs_stage_inversions} expand the examples in Figure~\ref{fig:channel_holography}(b) using five Stanford Dogs images with ConvNeXt-Base. The positive- and negative-weight groups retain recognizable scene structure, whereas their signed combination emphasizes more localized object structure. The stage-wise views show the corresponding scene content along the encoder hierarchy.

\begin{figure}[p]
\centering
\small
\begingroup
\setlength{\tabcolsep}{3pt}
\begin{tabular}{@{}cccc@{}}
\textbf{Input} & \shortstack{\textbf{Positive component}\\$\hat X^{(c)}_+$} & \shortstack{\textbf{Negative-weight group}\\$-\hat X^{(c)}_-$} & \shortstack{\textbf{Signed combination}\\$\hat X^{(c)}$}\\[4pt]
\includegraphics[width=0.220\linewidth,trim={7.200000bp 1063.368166bp 1625.880000bp 26.077899bp},clip]{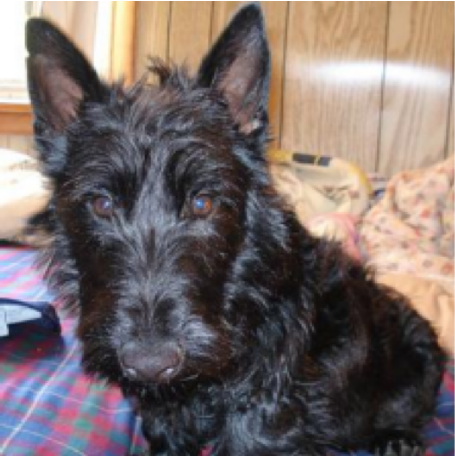} &
\includegraphics[width=0.220\linewidth,trim={1394.918765bp 1063.368166bp 238.161235bp 26.077899bp},clip]{figures/class_attribution_posneg.pdf} &
\includegraphics[width=0.220\linewidth,trim={1626.205226bp 1063.368166bp 6.874774bp 26.077899bp},clip]{figures/class_attribution_posneg.pdf} &
\includegraphics[width=0.220\linewidth,trim={1163.632304bp 1063.368166bp 469.447696bp 26.077899bp},clip]{figures/class_attribution_posneg.pdf}\\[3pt]
\includegraphics[width=0.220\linewidth,trim={7.200000bp 799.326124bp 1625.880000bp 290.119941bp},clip]{figures/class_attribution_posneg.pdf} &
\includegraphics[width=0.220\linewidth,trim={1394.918765bp 799.326124bp 238.161235bp 290.119941bp},clip]{figures/class_attribution_posneg.pdf} &
\includegraphics[width=0.220\linewidth,trim={1626.205226bp 799.326124bp 6.874774bp 290.119941bp},clip]{figures/class_attribution_posneg.pdf} &
\includegraphics[width=0.220\linewidth,trim={1163.632304bp 799.326124bp 469.447696bp 290.119941bp},clip]{figures/class_attribution_posneg.pdf}\\[3pt]
\includegraphics[width=0.220\linewidth,trim={7.200000bp 535.284083bp 1625.880000bp 554.161982bp},clip]{figures/class_attribution_posneg.pdf} &
\includegraphics[width=0.220\linewidth,trim={1394.918765bp 535.284083bp 238.161235bp 554.161982bp},clip]{figures/class_attribution_posneg.pdf} &
\includegraphics[width=0.220\linewidth,trim={1626.205226bp 535.284083bp 6.874774bp 554.161982bp},clip]{figures/class_attribution_posneg.pdf} &
\includegraphics[width=0.220\linewidth,trim={1163.632304bp 535.284083bp 469.447696bp 554.161982bp},clip]{figures/class_attribution_posneg.pdf}\\[3pt]
\includegraphics[width=0.220\linewidth,trim={7.200000bp 271.242041bp 1625.880000bp 818.204024bp},clip]{figures/class_attribution_posneg.pdf} &
\includegraphics[width=0.220\linewidth,trim={1394.918765bp 271.242041bp 238.161235bp 818.204024bp},clip]{figures/class_attribution_posneg.pdf} &
\includegraphics[width=0.220\linewidth,trim={1626.205226bp 271.242041bp 6.874774bp 818.204024bp},clip]{figures/class_attribution_posneg.pdf} &
\includegraphics[width=0.220\linewidth,trim={1163.632304bp 271.242041bp 469.447696bp 818.204024bp},clip]{figures/class_attribution_posneg.pdf}\\[3pt]
\includegraphics[width=0.220\linewidth,trim={7.200000bp 7.200000bp 1625.880000bp 1082.246065bp},clip]{figures/class_attribution_posneg.pdf} &
\includegraphics[width=0.220\linewidth,trim={1394.918765bp 7.200000bp 238.161235bp 1082.246065bp},clip]{figures/class_attribution_posneg.pdf} &
\includegraphics[width=0.220\linewidth,trim={1626.205226bp 7.200000bp 6.874774bp 1082.246065bp},clip]{figures/class_attribution_posneg.pdf} &
\includegraphics[width=0.220\linewidth,trim={1163.632304bp 7.200000bp 469.447696bp 1082.246065bp},clip]{figures/class_attribution_posneg.pdf}
\end{tabular}
\endgroup
\caption{Additional class-directional reconstructions on Stanford Dogs with ConvNeXt-Base, using all 1024 final-stage channels. Each row shows the input, positive-weight component, negative-weight group, and full signed reconstruction. The negative-weight group is rendered with absolute classifier weights, giving $-\hat X^{(c)}_-$. The signed combination $\hat X^{(c)}=\hat X^{(c)}_++\hat X^{(c)}_-$ is formed before independently min--max rescaling the three reconstruction columns for display. Ground-truth and predicted class indices agree in all rows: 48, 39, 108, 59, and 26, from top to bottom. The fourth row also appears in Figure~\ref{fig:channel_holography}(b).}
\label{fig:dogs_signed_composition}
\end{figure}

\begin{figure}[p]
\centering
\small
\begingroup
\setlength{\tabcolsep}{3pt}
\begin{tabular}{@{}cccc@{}}
\textbf{Stage 0} & \textbf{Stage 1} & \textbf{Stage 2} & \textbf{Stage 3}\\[4pt]
\includegraphics[width=0.220\linewidth,trim={238.486461bp 1063.368166bp 1394.593539bp 26.077899bp},clip]{figures/class_attribution_posneg.pdf} &
\includegraphics[width=0.220\linewidth,trim={469.772922bp 1063.368166bp 1163.307078bp 26.077899bp},clip]{figures/class_attribution_posneg.pdf} &
\includegraphics[width=0.220\linewidth,trim={701.059382bp 1063.368166bp 932.020618bp 26.077899bp},clip]{figures/class_attribution_posneg.pdf} &
\includegraphics[width=0.220\linewidth,trim={932.345843bp 1063.368166bp 700.734157bp 26.077899bp},clip]{figures/class_attribution_posneg.pdf}\\[3pt]
\includegraphics[width=0.220\linewidth,trim={238.486461bp 799.326124bp 1394.593539bp 290.119941bp},clip]{figures/class_attribution_posneg.pdf} &
\includegraphics[width=0.220\linewidth,trim={469.772922bp 799.326124bp 1163.307078bp 290.119941bp},clip]{figures/class_attribution_posneg.pdf} &
\includegraphics[width=0.220\linewidth,trim={701.059382bp 799.326124bp 932.020618bp 290.119941bp},clip]{figures/class_attribution_posneg.pdf} &
\includegraphics[width=0.220\linewidth,trim={932.345843bp 799.326124bp 700.734157bp 290.119941bp},clip]{figures/class_attribution_posneg.pdf}\\[3pt]
\includegraphics[width=0.220\linewidth,trim={238.486461bp 535.284083bp 1394.593539bp 554.161982bp},clip]{figures/class_attribution_posneg.pdf} &
\includegraphics[width=0.220\linewidth,trim={469.772922bp 535.284083bp 1163.307078bp 554.161982bp},clip]{figures/class_attribution_posneg.pdf} &
\includegraphics[width=0.220\linewidth,trim={701.059382bp 535.284083bp 932.020618bp 554.161982bp},clip]{figures/class_attribution_posneg.pdf} &
\includegraphics[width=0.220\linewidth,trim={932.345843bp 535.284083bp 700.734157bp 554.161982bp},clip]{figures/class_attribution_posneg.pdf}\\[3pt]
\includegraphics[width=0.220\linewidth,trim={238.486461bp 271.242041bp 1394.593539bp 818.204024bp},clip]{figures/class_attribution_posneg.pdf} &
\includegraphics[width=0.220\linewidth,trim={469.772922bp 271.242041bp 1163.307078bp 818.204024bp},clip]{figures/class_attribution_posneg.pdf} &
\includegraphics[width=0.220\linewidth,trim={701.059382bp 271.242041bp 932.020618bp 818.204024bp},clip]{figures/class_attribution_posneg.pdf} &
\includegraphics[width=0.220\linewidth,trim={932.345843bp 271.242041bp 700.734157bp 818.204024bp},clip]{figures/class_attribution_posneg.pdf}\\[3pt]
\includegraphics[width=0.220\linewidth,trim={238.486461bp 7.200000bp 1394.593539bp 1082.246065bp},clip]{figures/class_attribution_posneg.pdf} &
\includegraphics[width=0.220\linewidth,trim={469.772922bp 7.200000bp 1163.307078bp 1082.246065bp},clip]{figures/class_attribution_posneg.pdf} &
\includegraphics[width=0.220\linewidth,trim={701.059382bp 7.200000bp 932.020618bp 1082.246065bp},clip]{figures/class_attribution_posneg.pdf} &
\includegraphics[width=0.220\linewidth,trim={932.345843bp 7.200000bp 700.734157bp 1082.246065bp},clip]{figures/class_attribution_posneg.pdf}
\end{tabular}
\endgroup
\caption{Stage-wise inversions for the same five Stanford Dogs images and ConvNeXt-Base model as Figure~\ref{fig:dogs_signed_composition}, in the same row order. Columns show reconstructions from stages 0--3, using up to 512 channels per stage ranked by batch-mean activation energy. Scene layout remains visible across stages as texture and fine detail change. Stage views use the display mapping $\operatorname{clip}(0.5\hat X_l+0.5,0,1)$.}
\label{fig:dogs_stage_inversions}
\end{figure}
\FloatBarrier

\subsection{Shared-Scene Superposition and Destructive Interference}
\label{app:exp_C4}

\paragraph{Operational definition.}
We use \emph{shared-scene superposition} to describe channel inversions that express common input-scene content together with channel-specific residuals. The term \emph{holographic} refers to the global scene structure visible in individual channel inversions. This definition connects the common-component decomposition of H1 to the images and regional measurements, while H2 describes how classifier weights combine these components. It provides an image-space account of shared content and channel differences, related to the broader superposition perspective of~\cite{elhage2022toy}.

\paragraph{Per-image and averaged-Gram statistics.}
For an image $X$, form
\begin{equation}
V(X)=[\operatorname{vec}(\tilde V_{L-1,1}(X)),\ldots,\operatorname{vec}(\tilde V_{L-1,C}(X))],
\qquad G(X)=V(X)^\top V(X),
\end{equation}
where $C=C_{L-1}$. The two aggregation orders give distinct statistics:
\begin{equation}\label{eq:scene-evr-aggregation}
\mathrm{EVR}_{\mathrm{avgGram}}
=\frac{\lambda_1(\bar G)}{\operatorname{tr}(\bar G)},
\qquad \bar G=\frac1N\sum_{n=1}^N G(X_n),
\qquad
\overline{\mathrm{EVR}}_1
=\frac1N\sum_{n=1}^N
\frac{\lambda_1(G(X_n))}{\operatorname{tr}(G(X_n))}.
\end{equation}
The averaged-Gram statistic is $20.9\%$. The per-image statistic measures the common direction separately for each input and has mean $77.4\%$ on CUB/ConvNeXt-Base and $95.9\%$ on Pets/ResNet-50 (Table~\ref{tab:scene_spectrum}).

To examine the contribution of shared channel content, channel-centering subtracts the mean inversion from each column:
\begin{equation}\label{eq:scene-channel-centering}
V_{\mathrm{ctr}}(X)=V(X)\left(I_C-\frac{\mathbf1\mathbf1^\top}{C}\right),
\qquad G_{\mathrm{ctr}}(X)=V_{\mathrm{ctr}}(X)^\top V_{\mathrm{ctr}}(X).
\end{equation}
We compare the spectra with four matched null controls: spatial permutation, phase randomization, across-image shuffling, and matched Gaussian signals. Table~\ref{tab:scene_spectrum} gives the largest null value for each centering condition. The uncentered leading-mode shares exceed the null values by a wide margin. After channel-centering, CUB does not exceed the strongest null and Pets is modestly above it. The strongest spectral evidence therefore concerns the common scene component, consistent with the full-scene inversions.

\begin{table}[t]
\centering\small
\setlength{\tabcolsep}{5pt}
\caption{Per-image leading-mode energy fractions. Brackets denote $95\%$ confidence intervals. The final column gives the strongest matched null before and after channel-centering.}
\label{tab:scene_spectrum}
\begin{tabular}{lccc}
\toprule
Dataset/backbone & Mean EVR@1 & Channel-centered & \shortstack{Strongest null\\uncentered / centered} \\
\midrule
CUB / ConvNeXt-B & $0.774\ [0.762,0.785]$ & $0.134\ [0.128,0.140]$ & $0.152\;/\;0.147$ \\
Pets / ResNet-50 & $0.959\ [0.955,0.962]$ & $0.241\ [0.230,0.251]$ & $0.197\;/\;0.201$ \\
\bottomrule
\end{tabular}
\end{table}

\paragraph{Shared direction and channel residuals.}
If $(\lambda_1,q_1)$ is the leading eigenpair of $G(X)$, with $\|q_1\|_2=1$ and $\lambda_1>0$, its corresponding unit direction in pixel space is
\begin{equation}\label{eq:scene-pixel-direction}
B(X)=\frac{V(X)q_1}{\sqrt{\lambda_1}},\qquad
C_i(X)=\langle\tilde V_{L-1,i}(X),B(X)\rangle
=\sqrt{\lambda_1}(q_1)_i.
\end{equation}
The residual $\delta_i(X)=\tilde V_{L-1,i}(X)-C_i(X)B(X)$ is orthogonal to $B(X)$, giving the H1 decomposition. A rank-one removal analysis yields an inter-channel variance statistic of $1.27\%$ of the average channel energy, compared with $0.0001\%$ before removal, a factor of $1.27\times10^4$. The pairwise cosine-similarity range widens from $[0.18,0.21]$ to $[-0.07,0.67]$.

\paragraph{Spatial entropy under untrained and sparse controls.}
We examine the effect of LAC optimization using an untrained LAC, five random-affine initializations with $\log\gamma,\beta\sim\mathcal N(0,0.1^2)$, and LACs trained with weak, medium, or strong sparsity weights $\lambda\in\{0.1,1,10\}$. For a nonzero spatially centered channel inversion flattened as $r\in\mathbb R^D$, the per-channel sparsity penalty is
\begin{equation}\label{eq:lac-sparsity-control}
\mathcal L_{\mathrm{sparse}}(r)=\frac{\|r\|_1}{\sqrt D\,\|r\|_2}.
\end{equation}
The sparse settings add this penalty to the reconstruction objective. To measure spatial concentration, let $e_{uv}$ be the RGB squared-energy map of a channel inversion and normalize it over its $H W$ pixels:
\begin{equation}\label{eq:lac-spatial-entropy}
p_{uv}=\frac{e_{uv}}{\sum_{u',v'}e_{u'v'}},\qquad
\mathcal H=-\frac{\sum_{u,v}p_{uv}\log p_{uv}}{\log(HW)},
\end{equation}
with $0\log0=0$. Each image contributes the median entropy across its channels. The controls use $200$ CUB images and $185$ Pets images, with images as the statistical unit. Table~\ref{tab:lac_entropy} summarizes the entropy measurements; the random-affine variability is the standard deviation across five initializations.

\begin{table}[t]
\centering\small
\setlength{\tabcolsep}{5pt}
\caption{Spatial entropy under different LAC settings. The weak, medium, and strong sparse settings use $\lambda=0.1,1,10$. Random-affine entries give mean$\pm$standard deviation across five initializations.}
\label{tab:lac_entropy}
\begin{tabular}{lcc}
\toprule
Setting & CUB / ConvNeXt-B & Pets / ResNet-50 \\
\midrule
Standard trained & $0.894$ & $0.899$ \\
Untrained & $0.869$ & $0.858$ \\
Random affine & $0.870\pm0.014$ & $0.862\pm0.001$ \\
Sparse, weak & $0.739$ & $0.826$ \\
Sparse, medium & $0.595$ & $0.532$ \\
Sparse, strong & $0.605$ & $0.541$ \\
\bottomrule
\end{tabular}
\end{table}

Spatial diffuseness persists in the untrained and random-affine controls, whereas their reconstruction fidelity is much lower (Table~\ref{tab:lac_control_fidelity}). For example, CUB SSIM is $0.378$ with trained LAC, compared with $0.014$ and $0.015$ for the untrained and random-affine settings. Spatial entropy and image fidelity therefore capture different properties: the breadth of energy distribution and the accuracy of input reconstruction.

\begin{table}[t]
\centering\small
\setlength{\tabcolsep}{9pt}
\caption{Reconstruction fidelity for the entropy controls. SSIM is higher-is-better and LPIPS is lower-is-better; values use their original scales. LPIPS is shown for the standard and sparsity-trained settings.}
\label{tab:lac_control_fidelity}
\begin{tabular}{lcccc}
\toprule
& \multicolumn{2}{c}{CUB / ConvNeXt-B} & \multicolumn{2}{c}{Pets / ResNet-50} \\
\cmidrule(lr){2-3}\cmidrule(lr){4-5}
Setting & SSIM $\uparrow$ & LPIPS $\downarrow$ & SSIM $\uparrow$ & LPIPS $\downarrow$ \\
\midrule
Standard trained & $0.378$ & $0.661$ & $0.473$ & $0.525$ \\
Untrained & $0.014$ & --- & $0.068$ & --- \\
Random affine & $0.015$ & --- & $0.046$ & --- \\
Sparse, weak & $0.472$ & $0.603$ & $0.452$ & $0.539$ \\
Sparse, medium & $0.422$ & $0.824$ & $0.317$ & $0.849$ \\
Sparse, strong & $0.406$ & $0.851$ & $0.292$ & $0.853$ \\
\bottomrule
\end{tabular}
\end{table}

Weak sparsity concentrates spatial energy while preserving comparable reconstruction quality: CUB improves in both metrics, whereas Pets shows a small decrease in SSIM and increase in LPIPS. Medium and strong penalties reduce entropy further and incur a larger perceptual-fidelity cost. These controls establish that broad energy distribution is present before affine optimization, while its extent and the fidelity of scene reconstruction depend on the training objective. Standard trained LAC combines broad scene coverage with input-aligned reconstruction.

\paragraph{H2: signed class-weighted synthesis.}
For fixed $X$, the channel inversions form a fixed synthesis basis. The mapping from coefficients to reconstructed pixels is linear:
\begin{equation}
\label{eq:adjoint-linearity-app}
\mathcal R_X(a)=\sum_i a_i\tilde V_{L-1,i}(X),\qquad
\hat X^{(c)}=\mathcal R_X(w^{(c)}).
\end{equation}
Substituting H1 gives the exact decomposition
\begin{equation}
\label{eq:h2-app}
\hat X^{(c)}(X)=B(X)\,\underbrace{\sum_iw_i^{(c)}C_i(X)}_{\alpha^{(c)}(X)}
+\underbrace{\sum_iw_i^{(c)}\delta_i(X)}_{\text{class-weighted residual}}.
\end{equation}
The coefficient $\alpha^{(c)}$ controls the shared contribution, while the second term combines channel-specific structure. The observed foreground concentration is described by
\begin{equation}
\begin{aligned}
\text{(B1)}\quad &\alpha^{(c)}(X)=\sum_iw_i^{(c)}C_i(X)\approx0,\\
\text{(B2)}\quad &\sum_iw_i^{(c)}\delta_i(X)\text{ concentrates on class-}c\text{ structure}.
\end{aligned}
\end{equation}
Relation (B1) describes destructive interference of shared content, and (B2) describes the discriminative structure retained in the combination.

Each sign-restricted component has the same decomposition:
\begin{equation}
\hat X^{(c)}_\pm=B(X)\sum_{i:\,\operatorname{sign}(w_i^{(c)})=\pm}w_i^{(c)}C_i(X)
+\sum_{i:\,\operatorname{sign}(w_i^{(c)})=\pm}w_i^{(c)}\delta_i(X).
\end{equation}
Common scene content remains prominent in both components. When their shared coefficients cancel, the full reconstruction emphasizes the class-weighted residual. In a correct-class analysis on CUB/ConvNeXt-B, the full signed reconstruction has foreground-energy proportion $0.722$, compared with $0.425$ for each sign-restricted component. This directly measures foreground concentration emerging from signed combination. Together with the channel interventions, the decomposition connects shared scene content and regional preferences to the classifier's behavior.

\subsection{Proof of Theorem~\ref{thm:covvol_duality} (Covariance Volume of Linear Responses)}
\label{app:exp_C5}

This appendix proves the response-volume interpretation of CovVol, derives its greedy selection rule, and relates its geometric motivation to H1/H2.

\paragraph{Centered channel responses.}
Let $H\in\mathbb{R}^{N\times C_{L-1}}$ contain the GAP features of $N$ images, with $H_{n,i}=h_i(X_n)$. Define
\begin{equation}
A=\frac{1}{\sqrt N}\bigl(H-\mathbf1\bar h^\top\bigr),
\qquad
\bar h=\frac{1}{N}H^\top\mathbf1,
\qquad
\Sigma=A^\top A.
\end{equation}
For a channel subset $S$ of size $k$, the columns of $A_S\in\mathbb{R}^{N\times k}$ are the selected channels' centered responses across the dataset. If a linear head uses coefficients $w\in\mathbb{R}^k$, then $A_Sw$ is its vector of centered scores divided by $\sqrt N$. Thus $A_S$ maps coefficient space to linear responses on the dataset.

\paragraph{Response-volume theorem.}
The squared $k$-dimensional volume of the parallelepiped generated by these response vectors is
\begin{equation}
\label{eq:covvol-duality-app}
\operatorname{Vol}_k\!\bigl(A_S[0,1]^k\bigr)^2
=\det(A_S^\top A_S)
=\det(\Sigma_S).
\end{equation}
Equivalently, for the Euclidean unit coefficient ball $\mathbb B_2^k$,
\begin{equation}
\operatorname{Vol}_k\!\bigl(A_S\mathbb B_2^k\bigr)
=\kappa_k\sqrt{\det(\Sigma_S)},
\end{equation}
where $\kappa_k=\operatorname{Vol}_k(\mathbb B_2^k)$. When $A_S$ has rank less than $k$, both $k$-dimensional volumes are zero.

\begin{proof}
For full-column-rank $A_S$, take a thin singular value decomposition $A_S=P\Lambda Q^\top$, where the columns of $P$ and $Q$ are orthonormal and the diagonal entries of $\Lambda$ are the positive singular values $s_1,\ldots,s_k$. The orthonormal factors preserve $k$-dimensional volume, while $\Lambda$ multiplies it by $\prod_j s_j$. Also,
\begin{equation}
\det(\Sigma_S)=\det(A_S^\top A_S)=\prod_{j=1}^k s_j^2.
\end{equation}
The unit cube has volume one and the unit ball has volume $\kappa_k$, giving both identities. If the rank is less than $k$, the images lie in a lower-dimensional subspace and the Gram determinant is zero.
\end{proof}

CovVol therefore maximizes the volume of a bounded family of linear responses. This objective combines response amplitude and directional diversity: scaling a channel changes the volume, while linear dependence collapses it. At fixed $k$, maximizing the determinant is equivalent to maximizing either volume above.

\paragraph{Exact greedy increments.}
Suppose the selected columns $A_S$ are linearly independent. For an unselected channel $i$, the block determinant identity gives
\begin{equation}
\label{eq:covvol-schur}
\frac{\det(\Sigma_{S\cup\{i\}})}{\det(\Sigma_S)}
=\Sigma_{ii}-\Sigma_{iS}\Sigma_S^{-1}\Sigma_{Si}
=\bigl\|(I-P_S)A_{:,i}\bigr\|_2^2,
\end{equation}
where $P_S=A_S(A_S^\top A_S)^{-1}A_S^\top$ is the orthogonal projector onto the selected response span. For $S=\varnothing$, use $P_S=0$ and $\det(\Sigma_\varnothing)=1$. Thus the channel with largest unexplained response variance gives the largest determinant after the next addition. When the residual is positive, the corresponding increment in $\log\det$ is the logarithm of the right-hand side.

Let $R=A^\top(I-P_S)A$. After selecting a channel $i$ with $R_{ii}>0$, orthogonally removing its residual response gives
\begin{equation}
R\ \leftarrow\ R-\frac{R_{:,i}R_{i,:}}{R_{ii}}.
\end{equation}
This is the pivoted Cholesky update in Algorithm~\ref{alg:covvol}. Starting from $R=\Sigma$, the algorithm therefore implements the greedy determinant rule in exact arithmetic. If all remaining residuals are zero, the response span has been exhausted; every larger subset has zero determinant. Positive pivots are available through $k\le\operatorname{rank}(A)$.

\paragraph{Residual geometry under background cancellation.}
H1/H2 supply the mechanism-level motivation for selecting channels with substantial amplitude and complementary directions. For a fixed input and subset $S$, write $\Delta_S=[\delta_i]_{i\in S}$ and let $C_S$ collect the shared-background coefficients. The class-directional synthesis obeys
\begin{equation}
\sum_{i\in S}w_i\widetilde V_{L-1,i}
=B\,C_S^\top w+\Delta_Sw.
\end{equation}
Under the cancellation constraint $C_S^\top w=0$, the reachable residuals form $\mathcal F_S=\{\Delta_Sw:C_S^\top w=0\}$. For $k\ge2$ and $C_S\ne0$, choose an orthonormal basis $U\in\mathbb{R}^{k\times(k-1)}$ of $C_S^\perp$. The bounded coefficient family $\{w:C_S^\top w=0,\ \|w\|_2\le1\}$ maps to $\Delta_SU\mathbb B_2^{k-1}$, whose volume is
\begin{equation}
\label{eq:vol-def-app}
\operatorname{Vol}_{k-1}\!\bigl(\Delta_SU\mathbb B_2^{k-1}\bigr)
=\kappa_{k-1}\sqrt{\det(U^\top G_\delta U)},
\qquad G_\delta=\Delta_S^\top\Delta_S.
\end{equation}
The determinant is independent of the orthonormal basis $U$. It is positive precisely when $\Delta_SU$ has rank $k-1$; otherwise the image has lower dimension. This measures residual volume expansion under a unit coefficient-norm constraint.

When $G_\delta\succ0$, the restricted determinant has the useful form
\begin{equation}
\det(U^\top G_\delta U)
=\det(G_\delta)\,
\frac{C_S^\top G_\delta^{-1}C_S}{\|C_S\|_2^2}.
\end{equation}
To see this, complete $U$ to the orthogonal matrix $Q=[U,c]$ with $c=C_S/\|C_S\|_2$. The last diagonal entry of $(Q^\top G_\delta Q)^{-1}$ equals $c^\top G_\delta^{-1}c$; its cofactor formula is $\det(U^\top G_\delta U)/\det(G_\delta)$. If $C_S=0$, cancellation imposes no constraint and the full coefficient ball maps through $\Delta_S$ instead.

\paragraph{From interference geometry to CovVol.}
The residual construction describes how coefficient combinations cancel shared scene content and retain diverse image-space structure. CovVol measures the amplitude and diversity of the same channels' forward responses across images, retaining their natural scales before LAC normalization. Equation~\eqref{eq:covvol-duality-app} gives this feature-space criterion an exact response-volume interpretation, and Eq.~\eqref{eq:covvol-schur} yields its efficient greedy implementation. This connects the interference analysis to a channel-selection rule using only the frozen encoder's GAP features; Section~\ref{sec:covvol_validation} evaluates the resulting compressed representations.

\subsection{OOD Validation: Numerical Tables and Detailed Analysis}
\label{app:exp_C6}

This appendix expands Section~\ref{sec:covvol_validation} with full numerical results and detailed analysis of the three findings.

\paragraph{Full numerical tables.}
Tables~\ref{tab:ood_id}, \ref{tab:rn18_ood_id}, \ref{tab:rn50_ood_id}, and \ref{tab:dn121_ood_id} report CIFAR-100 (in-distribution) and CIFAR-100-C (corruption-OOD) accuracy across all pruning ratios and selection methods for ConvNeXt-Base, ResNet-18, ResNet-50, and DenseNet-121, respectively. Tables~\ref{tab:ood_xfer}, \ref{tab:rn18_ood_xfer}, \ref{tab:rn50_ood_xfer}, and \ref{tab:dn121_ood_xfer} report cross-dataset transfer results on Stanford Dogs and Oxford-IIIT Pet. All experiments use the same frozen models.

\begin{table}[!ht]
\centering\small
\caption{ConvNeXt-Base: In-distribution and corruption-OOD accuracy.}
\label{tab:ood_id}
\setlength{\tabcolsep}{4pt}
\begin{tabular}{c|ccc|ccc}
\toprule
& \multicolumn{3}{c|}{CIFAR-100 (ID)} & \multicolumn{3}{c}{CIFAR-100-C (OOD)} \\
Prune\% & CovVol & Energy & Random & CovVol & Energy & Random \\
\midrule
0  & 88.23 & 88.23 & $88.23{\pm}0.00$ & 62.77 & 62.77 & $62.77{\pm}0.00$ \\
50 & 88.12 & 88.20 & $88.19{\pm}0.02$ & 62.60 & 62.60 & $62.60{\pm}0.09$ \\
80 & 88.01 & 87.91 & $88.02{\pm}0.12$ & 62.04 & 62.08 & $62.08{\pm}0.15$ \\
90 & 87.50 & 87.44 & $87.51{\pm}0.12$ & 61.23 & 61.18 & $61.18{\pm}0.20$ \\
95 & 86.61 & 86.32 & $86.26{\pm}0.33$ & 59.44 & 59.03 & $59.03{\pm}0.37$ \\
98 & \textbf{82.20} & 80.17 & $80.72{\pm}0.46$ & \textbf{51.96} & 48.98 & $50.99{\pm}0.63$ \\
\bottomrule
\end{tabular}
\end{table}

\begin{table}[!ht]
\centering\small
\caption{ConvNeXt-Base: Cross-dataset transfer on Stanford Dogs and Oxford-IIIT Pet.}
\label{tab:ood_xfer}
\setlength{\tabcolsep}{4pt}
\begin{tabular}{c|ccc|ccc}
\toprule
& \multicolumn{3}{c|}{Stanford Dogs} & \multicolumn{3}{c}{Oxford-IIIT Pet} \\
Prune\% & CovVol & Energy & Random & CovVol & Energy & Random \\
\midrule
0  & 91.75 & 91.75 & $91.75{\pm}0.00$ & 94.39 & 94.39 & $94.39{\pm}0.00$ \\
50 & 91.99 & 91.92 & $91.72{\pm}0.15$ & 94.03 & 94.25 & $94.43{\pm}0.20$ \\
80 & 91.74 & 91.67 & $91.54{\pm}0.13$ & 93.79 & 93.89 & $93.81{\pm}0.18$ \\
90 & 91.36 & 91.31 & $91.18{\pm}0.12$ & 93.46 & 93.38 & $93.12{\pm}0.12$ \\
95 & 90.63 & 90.49 & $90.00{\pm}0.14$ & 91.53 & 91.59 & $91.47{\pm}0.33$ \\
98 & \textbf{86.93} & 85.49 & $85.12{\pm}0.11$ & \textbf{86.43} & 85.09 & $85.31{\pm}0.53$ \\
\bottomrule
\end{tabular}
\end{table}

\begin{table}[!ht]
\centering\small
\caption{ResNet-18: In-distribution and corruption-OOD accuracy.}
\label{tab:rn18_ood_id}
\setlength{\tabcolsep}{4pt}
\begin{tabular}{c|ccc|ccc}
\toprule
& \multicolumn{3}{c|}{CIFAR-100 (ID)} & \multicolumn{3}{c}{CIFAR-100-C (OOD)} \\
Prune\% & CovVol & Energy & Random & CovVol & Energy & Random \\
\midrule
0  & 78.09 & 78.09 & $78.09{\pm}0.00$ & 49.40 & 49.40 & $49.40{\pm}0.00$ \\
50 & 77.88 & 77.68 & $77.80{\pm}0.15$ & 49.14 & 49.01 & $49.12{\pm}0.04$ \\
80 & 77.20 & 76.69 & $76.98{\pm}0.17$ & 48.43 & 47.92 & $48.11{\pm}0.09$ \\
90 & 75.43 & 75.29 & $75.11{\pm}0.19$ & 46.71 & 45.92 & $46.10{\pm}0.22$ \\
95 & 67.96 & 66.62 & $67.25{\pm}0.42$ & 38.57 & 37.48 & $38.16{\pm}0.27$ \\
98 & \textbf{50.79} & 49.59 & $49.53{\pm}0.51$ & \textbf{26.59} & 25.46 & $25.87{\pm}0.31$ \\
\bottomrule
\end{tabular}
\end{table}

\begin{table}[!ht]
\centering\small
\caption{ResNet-18: Cross-dataset transfer on Stanford Dogs and Oxford-IIIT Pet.}
\label{tab:rn18_ood_xfer}
\setlength{\tabcolsep}{4pt}
\begin{tabular}{c|ccc|ccc}
\toprule
& \multicolumn{3}{c|}{Stanford Dogs} & \multicolumn{3}{c}{Oxford-IIIT Pet} \\
Prune\% & CovVol & Energy & Random & CovVol & Energy & Random \\
\midrule
0  & 78.62 & 78.62 & $78.62{\pm}0.00$ & 89.32 & 89.32 & $89.32{\pm}0.00$ \\
50 & 77.28 & 76.66 & $76.93{\pm}0.29$ & 87.16 & 86.94 & $87.64{\pm}0.52$ \\
80 & 73.26 & 72.44 & $73.35{\pm}0.40$ & 83.32 & 83.21 & $84.25{\pm}0.31$ \\
90 & 68.08 & 67.55 & $65.30{\pm}1.91$ & 77.57 & 79.72 & $78.64{\pm}1.07$ \\
95 & 52.76 & 47.89 & $42.81{\pm}2.69$ & 68.30 & 64.90 & $60.41{\pm}3.61$ \\
98 & \textbf{25.99} & 23.82 & $22.49{\pm}2.88$ & \textbf{45.54} & 43.50 & $40.39{\pm}3.08$ \\
\bottomrule
\end{tabular}
\end{table}

\begin{table}[!ht]
\centering\small
\caption{ResNet-50: In-distribution and corruption-OOD accuracy.}
\label{tab:rn50_ood_id}
\setlength{\tabcolsep}{4pt}
\begin{tabular}{c|ccc|ccc}
\toprule
& \multicolumn{3}{c|}{CIFAR-100 (ID)} & \multicolumn{3}{c}{CIFAR-100-C (OOD)} \\
Prune\% & CovVol & Energy & Random & CovVol & Energy & Random \\
\midrule
0  & 79.23 & 79.23 & $79.23{\pm}0.00$ & 49.83 & 49.83 & $49.83{\pm}0.00$ \\
50 & 79.24 & 79.18 & $79.16{\pm}0.05$ & 49.80 & 49.74 & $49.80{\pm}0.03$ \\
80 & 79.38 & 79.26 & $79.10{\pm}0.11$ & 49.81 & 49.46 & $49.68{\pm}0.07$ \\
90 & 78.83 & 78.60 & $78.61{\pm}0.15$ & 49.54 & 49.04 & $49.32{\pm}0.12$ \\
95 & 77.58 & 77.12 & $77.24{\pm}0.15$ & 48.19 & 47.28 & $47.51{\pm}0.28$ \\
98 & \textbf{75.55} & 74.22 & $75.14{\pm}0.26$ & \textbf{45.43} & 44.16 & $45.04{\pm}0.35$ \\
\bottomrule
\end{tabular}
\end{table}

\begin{table}[!ht]
\centering\small
\caption{ResNet-50: Cross-dataset transfer on Stanford Dogs and Oxford-IIIT Pet.}
\label{tab:rn50_ood_xfer}
\setlength{\tabcolsep}{4pt}
\begin{tabular}{c|ccc|ccc}
\toprule
& \multicolumn{3}{c|}{Stanford Dogs} & \multicolumn{3}{c}{Oxford-IIIT Pet} \\
Prune\% & CovVol & Energy & Random & CovVol & Energy & Random \\
\midrule
0  & 89.34 & 89.34 & $89.34{\pm}0.00$ & 90.76 & 90.76 & $90.76{\pm}0.00$ \\
50 & 88.64 & 89.04 & $88.59{\pm}0.10$ & 89.42 & 89.72 & $90.20{\pm}0.17$ \\
80 & 87.38 & 87.11 & $86.55{\pm}0.28$ & 88.39 & 87.98 & $88.15{\pm}0.32$ \\
90 & 85.29 & 85.05 & $84.48{\pm}0.40$ & 85.80 & 85.72 & $85.39{\pm}0.56$ \\
95 & 84.03 & 83.20 & $73.84{\pm}0.46$ & 82.86 & 81.95 & $79.22{\pm}1.37$ \\
98 & \textbf{81.18} & 80.05 & $54.93{\pm}3.01$ & \textbf{80.32} & 78.74 & $66.31{\pm}3.43$ \\
\bottomrule
\end{tabular}
\end{table}

\begin{table}[!ht]
\centering\small
\caption{DenseNet-121: In-distribution and corruption-OOD accuracy.}
\label{tab:dn121_ood_id}
\setlength{\tabcolsep}{4pt}
\begin{tabular}{c|ccc|ccc}
\toprule
& \multicolumn{3}{c|}{CIFAR-100 (ID)} & \multicolumn{3}{c}{CIFAR-100-C (OOD)} \\
Prune\% & CovVol & Energy & Random & CovVol & Energy & Random \\
\midrule
0  & 82.60 & 82.60 & $82.60{\pm}0.00$ & 53.32 & 53.32 & $53.32{\pm}0.00$ \\
50 & 82.53 & 82.53 & $82.35{\pm}0.13$ & 53.41 & 53.41 & $53.05{\pm}0.13$ \\
80 & 81.92 & 82.10 & $81.78{\pm}0.22$ & 52.79 & 52.58 & $51.93{\pm}0.15$ \\
90 & 81.30 & 80.90 & $80.02{\pm}0.34$ & 51.50 & 51.08 & $49.80{\pm}0.32$ \\
95 & 78.22 & 77.26 & $74.19{\pm}0.98$ & 47.50 & 46.44 & $43.49{\pm}0.97$ \\
98 & \textbf{70.00} & 67.19 & $59.59{\pm}3.32$ & \textbf{39.92} & 36.74 & $32.36{\pm}2.23$ \\
\bottomrule
\end{tabular}
\end{table}

\begin{table}[!ht]
\centering\small
\caption{DenseNet-121: Cross-dataset transfer on Stanford Dogs and Oxford-IIIT Pet.}
\label{tab:dn121_ood_xfer}
\setlength{\tabcolsep}{4pt}
\begin{tabular}{c|ccc|ccc}
\toprule
& \multicolumn{3}{c|}{Stanford Dogs} & \multicolumn{3}{c}{Oxford-IIIT Pet} \\
Prune\% & CovVol & Energy & Random & CovVol & Energy & Random \\
\midrule
0  & 78.99 & 78.99 & $78.99{\pm}0.00$ & 88.58 & 88.58 & $88.58{\pm}0.00$ \\
50 & 79.14 & 77.47 & $76.02{\pm}0.69$ & 88.42 & 88.14 & $87.29{\pm}0.52$ \\
80 & 77.18 & 73.60 & $68.00{\pm}1.54$ & 85.17 & 86.02 & $81.89{\pm}1.27$ \\
90 & 71.62 & 67.00 & $57.63{\pm}1.25$ & 81.79 & 80.98 & $74.68{\pm}1.36$ \\
95 & 56.62 & 50.02 & $35.79{\pm}1.99$ & 71.74 & 68.96 & $56.69{\pm}3.05$ \\
98 & \textbf{36.17} & 31.34 & $21.39{\pm}4.36$ & \textbf{57.32} & 50.83 & $41.30{\pm}4.94$ \\
\bottomrule
\end{tabular}
\end{table}

\paragraph{Detailed Finding 1: broad channel redundancy.}
ConvNeXt-Base retains strong performance through substantial channel removal. At $90\%$ pruning, CovVol changes accuracy from $88.23\%$ to $87.50\%$ on CIFAR-100, $62.77\%$ to $61.23\%$ on CIFAR-100-C, $91.75\%$ to $91.36\%$ on Dogs, and $94.39\%$ to $93.46\%$ on Pets. The three selection methods are close over this range. The extent of the plateau varies across backbones and tasks, as the full tables show. Because a linear probe is retrained on each retained subset, these results demonstrate that a small subset can retain much of the representation's classification utility. This broad redundancy complements the full-scene channel inversions underlying H1.

\paragraph{Detailed Finding 2: channel diversity matters at small budgets.}
At extreme compression, CovVol consistently outperforms Energy and Random in the reported $98\%$ pruning results across all four backbones and tasks. For ConvNeXt-Base on CIFAR-100, it achieves $82.20\%$, compared with $80.17\%$ for Energy and $80.72\%$ for Random; on CIFAR-100-C the corresponding values are $51.96$ vs. $48.98$ vs. $50.99$. The advantage supports selecting channels jointly according to their response geometry. Theorem~\ref{thm:covvol_duality} gives the precise forward-space interpretation: the covariance determinant measures the volume spanned by centered channel responses across samples. Energy selection ranks channels individually, whereas CovVol uses Schur-complement updates to favor high-variance responses that add directions beyond those already selected. This combination of amplitude and diversity is consistent with the interference account.

\paragraph{Detailed Finding 3: OOD degradation accompanies interference breakdown.}
Figure~\ref{fig:covvol_main}(c)(d) relates the reconstruction geometry to empirical failure under distribution shift. In panel (c), increasing corruption weakens the foreground silhouette and brings background content back into the reconstruction. Within H2, this is the visual signature of less effective cancellation and weaker foreground emergence as the input-dependent inversion basis changes.

Panel (d) shows the normalized evolution of classification accuracy, $\log\det(\Sigma)$, effective rank, $\mathrm{Tr}(\Sigma)$, and the feature-space MMD distance~\cite{gretton2012kernel} between clean and corrupted distributions. The Gaussian-kernel MMD is
\begin{equation}
\mathrm{MMD}^2=\mathbb{E}_{X,X'\sim P}[K(X,X')]-2\mathbb{E}_{X\sim P,Y\sim Q}[K(X,Y)]+\mathbb{E}_{Y,Y'\sim Q}[K(Y,Y')].
\end{equation}
Accuracy, log-determinant, effective rank, and trace decrease with corruption severity, while MMD increases. Their response magnitudes differ: accuracy and trace fall much more sharply than the normalized log-determinant and effective rank.

These joint trends connect loss of classification performance with changes in feature-response volume and diversity, alongside the visible breakdown of foreground concentration. Together they support the geometric interpretation of OOD degradation: distribution shift perturbs both the channel-response geometry and the input-dependent components combined by the classifier.

\subsection{CAM Comparison: Trade-off Analysis}
\label{app:exp_C7}

This appendix expands Section~\ref{sec:cam_comparison}.

\paragraph{Faithfulness and spatial detail.}
Table~\ref{tab:faithfulness} reports higher Insertion AUC and lower Deletion AUC for our method on both datasets. On CUB, our Deletion AUC is $0.164$, compared with $0.409$ for the strongest baseline, a reduction of approximately $60\%$. On Pets it is $0.384$, compared with $0.571$. Grad-CAM-family maps combine activations at feature-map resolution and interpolate them to the image grid. Our cascade instead propagates through the encoder's Jacobians with channel-wise LAC calibration, retaining finer spatial structure. The insertion and deletion results show that the resulting pixel ranking more effectively identifies evidence affecting the classifier's output in these evaluations.

\paragraph{Peak and ranking localization.}
Our method achieves the highest mAP and Pointing scores on both datasets. Pointing tests whether the highest-scoring pixel lies in the foreground mask; mAP evaluates the ranking of foreground pixels. The high scores agree with the foreground concentration of the class-directional reconstruction: combining the positive and negative components suppresses common content and emphasizes object-related structure.

\paragraph{Region overlap.}
mIoU measures overlap between a binarized attribution and the full foreground mask. Score-CAM achieves the highest mIoU on both datasets, while our method ranks second on Pets. Smooth CAM maps tend to cover broad object regions; our reconstructions emphasize more localized structures such as edges and sub-object details. These different spatial distributions explain the observed trade-off between peak or ranking localization and full-mask overlap. The CUB pAcc pattern similarly reflects a weaker match to the full segmentation mask.

\paragraph{Relation to the interference account.}
H2 describes how class-weighted channel differences become prominent after shared-content cancellation. The resulting emphasis on discriminative structures is consistent with strong insertion/deletion and peak-localization scores without uniform coverage of the foreground silhouette. Score-CAM illustrates the distinction between these evaluation criteria: it leads in mIoU, while its CUB Deletion AUC ($0.426$) is approximately $2.6$ times ours. The metrics therefore provide complementary views of class evidence and spatial coverage.

\Needspace{10\baselineskip}
\subsection{Causal ECR Ablation: Full Protocol and Analysis}
\label{app:exp_C8}

This appendix expands Section~\ref{sec:causal_ablation}.

\paragraph{ECR as a region-based channel score.}
Under H1, every channel inversion decomposes as $\tilde{V}_{L-1,i}(X)=C_i(X)B(X)+\delta_i(X)$. The ECR score measures the fraction of each inversion's squared energy inside a reference region:
\begin{equation}
\mathrm{ECR}_i=\frac{\sum_{(u,v)\in\mathcal{R}}\sum_k[C_iB+\delta_i]_k(u,v)^2}{\sum_{(u,v)}\sum_k[C_iB+\delta_i]_k(u,v)^2}.
\end{equation}
Thus, channels can share global scene support while receiving different scores according to their regional energy preferences. A high-ECR channel is a candidate carrier of structure associated with $\mathcal{R}$; channel ablation tests its contribution to classification.

\paragraph{Detailed experimental protocol.}
A single Birman image $X_{\mathrm{ref}}$ is selected from the test set, and its tail region $\mathcal{R}$ is manually annotated. From this annotation we compute $\{\mathrm{ECR}_i\}_{i=1}^{C_{L-1}}$ and freeze a ranking of all channels. We compare three ablation orders: \textbf{Descending} (high ECR first), \textbf{Ascending} (low ECR first), and \textbf{Random} (uniform random order, mean$\pm$std over 50 seeds). The frozen ranking is applied to every Birman test image. For each image $X$ and pruning fraction $\rho\in[0,1]$, we zero the top-$\lceil\rho C_{L-1}\rceil$ channels of $h(X)$ in the chosen order, re-run the classifier head, and record $P(\mathrm{Birman}\mid X,\rho)$, averaging across Birman test images.

\paragraph{Detailed numerical observations.}
The Descending curve falls faster than Random over the middle and high ablation ranges. At $40\%$ ablation, $P(\mathrm{Birman})$ is approximately $0.84$ for Descending, $0.90$ for Random, and $0.94$ for Ascending. At $70\%$, the corresponding values are $0.61$, $0.83$, and $0.88$. Because the channel ranking uses only the tail mask on one reference image, this intervention shows that the region-associated channels make a substantial contribution to Birman predictions across the test images.

Ascending retains higher confidence than Random through the middle of the ablation range, with a marked decline at high ablation fractions. At $100\%$ ablation, all three orders give the same zero-feature input to the classifier head and therefore converge to the same output. The maximum gap between Descending and Ascending exceeds $0.3$ in $P(\mathrm{Birman})$. This separation establishes that the ECR ranking identifies functionally different channel groups. It also reconciles the full-scene inversions with concept selectivity: global spatial support coexists with regional preferences that matter for classification.

\paragraph{Transfer of the single-image ranking.}
Channel indices are shared across images, so a ranking obtained from $X_{\mathrm{ref}}$ can be applied directly to the rest of the class. The separation of the resulting ablation curves demonstrates that the single-image ranking transfers to Birman test images. It provides a channel-level handle for linking a spatial annotation to a reproducible intervention on the frozen classifier.

\paragraph{Implications.}
The experiment connects pixel-space inversion to causal channel intervention using a single spatial annotation. Removing channels according to their regional inversion energy changes the classifier's output in a systematic order, supporting the functional relevance of the channel differences described by H1. ECR and ablation therefore provide a simple procedure for locating and testing region-associated channels, with applications in model auditing and concept-level debugging.

\subsection{Pure-ViT Inversion and Hybrid Attention Readouts}
\label{app:exp_C9}

This appendix gives the implementation and evaluation details for Section~\ref{sec:pure_vit}, followed by a separate experiment reusing the trained CNN cascade with a hybrid attention head.

\subsubsection{Pure-ViT adaptation and evaluation}
\label{app:pure_vit_details}

\paragraph{Correction placement and readouts.}
We implement channel-wise GroupNorm correction for a pure ViT-B/16, with LACs at Transformer block boundaries and a stem corrector in the backward cascade. These interfaces replace the stage interfaces used for CNN inversion. The selected readout modules in Figure~\ref{fig:pure_vit_inversion} specify the source representations to invert; the backward path uses the block-boundary correctors it traverses and the stem corrector. The B-prefixed labels identify the plotted readout modules. We evaluate trained ViT-LAC, untrained ViT-LAC, and raw VJP on the same 185 Pets test images, using the same top-128 energy dimensions in every condition. Table~\ref{tab:vit_inversion} reports metrics computed on output tensors without visual enhancement. The direct and contrast-adjusted panels in Figure~\ref{fig:pure_vit_inversion} show two display versions of the same inversions.

\paragraph{Centered correlation and contrast.}
These metrics separate alignment with the input from reconstruction contrast. Let $\mu_{\mathrm{RGB}}$ denote the global RGB mean removed in the evaluation, broadcast over the image. The centered reconstruction and input vectors are
\begin{equation}
\label{eq:vit-center}
y^\circ=\operatorname{vec}\bigl(\widehat X-\mu_{\mathrm{RGB}}(\widehat X)\bigr),
\qquad
x^\circ=\operatorname{vec}\bigl(X-\mu_{\mathrm{RGB}}(X)\bigr).
\end{equation}
For nonzero centered norms, the centered input correlation and output/input contrast are
\begin{equation}
\label{eq:vit-inversion-metrics}
\operatorname{Corr}(\widehat X,X)
=\frac{\langle y^\circ,x^\circ\rangle}{\|y^\circ\|_2\|x^\circ\|_2},
\qquad
\operatorname{Contrast}(\widehat X,X)
=\frac{\|y^\circ\|_2}{\|x^\circ\|_2}.
\end{equation}
We report the image-level mean of each metric with 95\% bootstrap confidence intervals using images as the resampling unit.

\paragraph{Signed-reconstruction diagnostic.}
We recombine ViT dimension inversions with signed class-head weights and compare the full reconstruction with its positive- and negative-weight partial sums using foreground energy. In this ViT-B/16/Pets setting, the full sum does not exhibit the foreground-energy gain over both partial sums observed in the CNN experiments.

\subsubsection{Reuse with a CNN-based hybrid attention head}
\label{app:hybrid_attention}

\paragraph{Hybrid architecture details.}
The frozen ConvNeXt-Base encoder produces $\mathbf{h}_{L-1}(X)\in\mathbb{R}^{1024\times7\times7}$. This tensor is flattened into 49 tokens of dimension 1024 and processed by three stacked multi-head self-attention blocks~\cite{NIPS2017_3f5ee243}, followed by a linear classification head. The attention blocks use LayerNorm and residual connections. Only the attention blocks and classifier head are trained; the convolutional encoder remains frozen at its ImageNet-pretrained weights. The LAC cascade trained in Section~\ref{sec:synthesis} is loaded unchanged for visualization.

\paragraph{Composition with the attention head.}
For a scalar query $\mathcal{Q}$ on the attention pathway, automatic differentiation gives the VJP $\partial\mathcal{Q}/\partial\mathbf{h}_{L-1}$. This serves as the seed for the trained LAC cascade:
\begin{equation}
\label{eq:hybrid-query-cascade}
\begin{split}
\widetilde{V}_{\mathcal Q}(X)
={}&\Psi_{\mathrm{LAC}_{\mathrm{stem}}}\circ J_{\mathrm{stem}}^T
\circ\Psi_{\mathrm{LAC}_0}\circ J_0^T\circ\cdots\\
&\circ\Psi_{\mathrm{LAC}_{L-1}}\circ J_{L-1}^T
\left(\frac{\partial\mathcal Q}{\partial\mathbf h_{L-1}}\right).
\end{split}
\end{equation}
The query changes the signal entering the cascade while retaining the same encoder Jacobians and learned LAC parameters. The local spatial action of each LAC remains a channel-wise calibration, as described in Section~\ref{sec:guarantees}.

\begin{figure}[!htbp]
\centering
\includegraphics[width=0.90\linewidth]{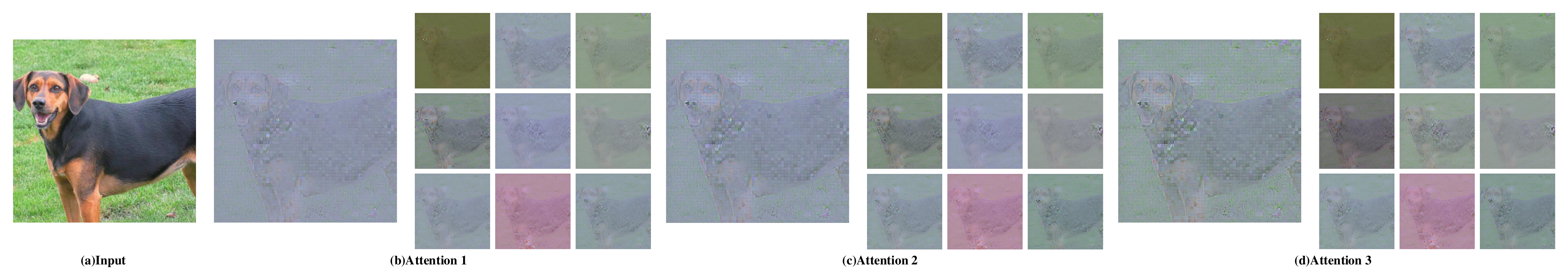}
\caption{Layer-wise (large) and per-token ($3{\times}3$) attention visualizations across three hybrid attention blocks above a frozen ConvNeXt-Base encoder. The trained CNN-side LAC cascade is reused for all readouts.}
\label{fig:vit_attention}
\end{figure}

\paragraph{Per-layer and per-token queries.}
Let $T_\ell$ denote the representation after the first $\ell$ attention blocks. For layer-wise visualization we use $\mathcal Q_\ell=\tfrac12\|T_\ell\|_F^2$, the total activation energy. For token $t\in\{1,\ldots,49\}$, we use $\mathcal Q_{\ell,t}=\tfrac12\|T_{\ell,t}\|_2^2$. Their VJPs supply the corresponding layer or token response for reconstruction through the same cascade.

Figure~\ref{fig:vit_attention} shows two patterns. The layer-wise visualizations remain coherent across all three attention depths, preserving the geometric layout of the scene. The per-token visualizations become more spatially selective with depth: early tokens give diffuse local evidence, while deeper tokens concentrate on sub-object regions. Adjacent tokens emphasize neighboring regions, reflecting the spatial organization learned above the frozen convolutional features.

The hybrid experiment demonstrates reuse of a trained CNN-side LAC cascade for new attention readouts without retraining its LAC parameters. The pure-ViT experiment instead adapts the correction placement to the native Transformer pathway.

\end{document}